\documentclass[journal, 10pt]{IEEEtran}

\usepackage{graphics} 
\usepackage{mathptmx} 
\usepackage{times} 
\usepackage{amsmath} 
\usepackage{amssymb} 
\usepackage{cite}
\usepackage{relsize} 
\usepackage{xcolor}
\usepackage{graphicx}
\usepackage{amsfonts}
\usepackage{mathtools}
\usepackage[mathcal]{euscript}
\usepackage{caption}
\usepackage[hidelinks]{hyperref}
\usepackage[T1]{fontenc} 
\usepackage[hyphenbreaks]{breakurl} 
\usepackage{mathtools, cuted} 
\usepackage[ruled,vlined]{algorithm2e}
\usepackage{enumitem} 
\usepackage{balance} 
\usepackage{booktabs}
\usepackage{tabularx}

\newcolumntype{x}[1]{%
>{\raggedleft\hspace{0pt}}p{#1}}%

\usepackage{subcaption}
\DeclareCaptionLabelSeparator{periodspace}{.\quad}
\usepackage{amsthm}

\newtheorem{lemma}{Lemma}
\newtheorem{remark}{Remark}

\newtheorem{theorem}{Theorem}
\newtheorem{corollary}{Corollary}

\theoremstyle{definition}
\newtheorem{example}{Example}

\def\bn{\mathbb N}

\def\br{\mathbb R}
\def\bc{\mathbb C}

\graphicspath{{Figs/}}

\begin{document}

\title{Robust Distributed Planar Formation Control for Higher-Order Holonomic and Nonholonomic Agents}

\author{Kaveh Fathian, Sleiman Safaoui, Tyler H. Summers, Nicholas R. Gans
\thanks{This work was supported by the U.S. Air Force Research Laboratory under grant FA8651-17-1-0001. The work of T. H. Summers was sponsored by the Army Research Office and was accomplished under Grant Number: W911NF-17-1-0058.}
\thanks{K. Fathian is with the Department of Aeronautics and Astronautics, Massachusetts Institute of Technology, Cambridge, MA, 02139 USA.~  E-mail: {kavehf@mit.edu}. S. Safaoui is with the Department of Electrical Engineering, T. H. Summers is with the  Department of Mechanical Engineering, University of Texas at Dallas, Richardson, TX, 75080 USA.~  E-mail: \{sleiman.safaoui, tyler.summers\}@utdallas.edu. N. R. Gans is with the UT Arlington Research Institute, University of Texas at Arlington, Arlington, TX, 76118 USA.~  E-mail: {nick.gans@uta.edu}.  }%
}%


\maketitle

\begin{abstract}
We present a distributed formation control strategy for agents with a variety of dynamics to achieve a desired planar formation.
Our approach is based on the barycentric-coordinate-based (BCB) control, 
which is fully distributed, does not require inter-agent communication or a common sense of orientation, and can be implemented using relative position measurements acquired by agents in their local coordinate frames. 
This removes the need for global positioning or alignment of local coordinate frames, which are required across several existing strategies.
We show how the BCB control for agents with the simplest dynamical model, i.e., the single-integrator dynamics, can be extended to agents with higher-order dynamics such as quadrotors, and nonholonomic agents such as unicycles and cars. 
Specifically, our extension preserves the desired convergence and robustness guarantees of the BCB approach and is provably
robust to saturations in the input and unmodeled linear actuator dynamics for unicycle and car agents.
We further show that under our proposed BCB control design, the agents can move along a rotated and scaled control direction without affecting the convergence to the desired formation. 
This observation is used to design a fully distributed collision avoidance strategy, which is often not considered in the formation control literature.
We demonstrate the proposed approach in simulations and further present a distributed robotic platform to test the strategy experimentally. Our experimental platform consists of off-the-shelf equipment that can be used to test and validate other multi-agent algorithms. The code and implementation instructions for this platform are available online. 
\end{abstract}

\begin{IEEEkeywords}
Multi-agent systems, formation control, distributed collision avoidance, distributed robotic platform.
\end{IEEEkeywords}

\IEEEdisplaynontitleabstractindextext

\IEEEpeerreviewmaketitle

\section*{Supplementary Material}

Video of paper summary, simulations, and experiments is available at 
{\href{https://youtu.be/1pfgXESMHxE}{https://youtu.be/1pfgXESMHxE}}.  Code for simulations and the distributed multi-robot platform can be download from
{\href{https://goo.gl/QH5qhw}{https://goo.gl/QH5qhw}}.

\section{Introduction}

Technological advances in recent years has made it increasingly possible to deploy a large fleet of agents to cooperatively map and monitor an environment \cite{Keller2017, Thomas2017}, deliver goods \cite{Dorling2017}, or manipulate objects \cite{Wang2016a, Saldana2017, Baehnemann2017}. 
In these applications, the ability to bring the agents to a desired geometric shape is a fundamental building block upon which more sophisticated maneuvering and navigation policies are constructed.
By assigning local control laws to individual agents, distributed formation control strategies ensure that a desired geometric shape emerge from the collective behavior of agents.
Compared to the centralized methods, distributed strategies have  better scalability, naturally parallelized computation, resilience to communication loss and hardware failure, and robustness to uncertainty and lack of global measurements.

\begin{figure}
	\begin{center}
		\includegraphics[trim =54mm 95mm 49mm 88mm, clip, width=0.5\textwidth]{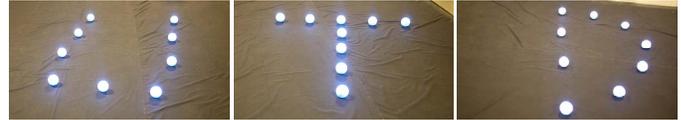}	
		\caption{The proposed formation control strategy implemented on our distributed robotic platform to form the letters UTD.}
		\label{fig:UTD}
	\end{center}
\end{figure}

In this work, we present a unified, distributed control strategy for planar formations of agents with a variety of dynamics. 
In particular, we consider agents with linear or input-to-state linearizable dynamics, and further extend the results to agents with nonholonomic unicycle and car dynamics.
Our approach is based on the barycentric-coordinate-based (BCB) control, 
which is fully distributed, does not require inter-agent communication or a common sense of orientation, and can be implemented using relative position measurements acquired by agents in their local coordinate frames. 
We start by formulating a semidefinite program (SDP) to compute the control gains needed for agents with the single-integrator model. 
Thanks to this design strategy, convergence to the desired formation is invariant to any (strictly) positive scaling of the control vector and any rotation amounting up to $\pm 90^\circ$. This observation leads to provable robustness guarantees such as robustness to saturations in the input and disturbances in the control direction.
This observation is further exploited to design a fully distributed collision avoidance strategy, which is often not considered in the formation control literature. 
The control for single-integrator agents is extended subsequently to agents with higher-order linear, feedback-to-state linearizable, nonholonomic unicycle, and nonholonomic car dynamics. 
The main challenges addressed in these extensions are to 1) ensure convergence guarantees to the desired formation are preserved; 2) ensure robustness properties are preserved (e.g., robustness to input saturations and unmodeled/unknown linear actuator dynamics); 3) have a unified design by using the same control gains computed from the SDP approach for single-integrator agents.
To vet the theoretical results, several simulations are presented for quadrotors, differential drive robots with unicycle dynamics, and cars, where it is shown that agents achieve a desired formation without collision. 
To typify the results further, the proposed control strategy is tested experimentally on a distributed differential-drive wheeled robotic platform with different numbers of robots and desired formations.

\subsection{Related Work and Contributions} \label{sec:contrib}

\phantomsection \label{p:b7} 
There exists a notable body of work on distributed formation control (see survey papers  \cite{Barca2013, Yan2013, Oh2015}). These works can be differentiated by the assumed sensing and measurements (e.g., global versus relative/local) and the use of inter-agent communication (e.g., allowed versus limited).
Examples of methods that require global measurements (e.g., GPS) are \cite{Michael2008, Hyun2016,  Motoyama2017}. 
Consensus-based methods, such as \cite{Lawton2003, Ren2007, Montijano2014, Lee2016}, or techniques based on distributed pose estimation \cite{Oh2012}, on the other hand, do not require global sensing. However, agents must communicate to peers during the mission to estimate their pose, synchronize their orientation, or register their local coordinate frames with respect to a common heading direction.

Unlike the aforementioned methods, certain class of formation control strategies are concerned with the most challenging case: when measurements are local/relative and inter-agent communication is limited or not allowed. Examples of the latter class include distance-based \cite{Olfati-Saber2002, Krick2009, Tian2013}, bearing-based \cite{Basiri2010, Deghat2015, Zhao2015, Trinh2018}, and the BCB formation control strategies \cite{Lin2016a, Bishop2013, Bishop2014, Fathian2016, Fathian2016a, Marina2017a, Han2018}.
Due to challenging nonlinear dynamics, no distance-based formation control algorithm with global convergence to the desired formation (in the general setting) is known to this day in the literature.
Moreover, bearing-based formation control methods with global convergence guarantees require alignment of local coordinate frames \cite{Zhao2019}.
Unlike the aforementioned methods, convergence guarantees of the BCB control to the desired shape are global (except for a measure zero set, which is inconsequential for the implementation).

The BCB control strategy was introduced by Lin et al. \cite{Lin2014, Lin2016}, who presented the general theory for agents with single-integrator dynamics and derived the (almost) \textit{global} convergence guarantees.
As these guarantees hold for agents with linear dynamics, it is therefore essential to extend them to agents with nonlinear and nonholonomic dynamics that are commonly encountered  in robotics applications.
In this paper, we leverage the gradient-descent control framework developed by Zhao et al. \cite{Zhao2017, Zhao2018} for agents with nonholonomic dynamics and show that for a subclass of sensing topologies that are undirected and universally rigid, the global convergence guarantees extend to agents with higher-order, input-to-state feedback linearizable, and nonholonomic unicycle and car dynamics.
We further show that under the proposed SDP design, robustness guarantees of the BCB control for single-integrator agents extend to agents with unicycle and car dynamics, and the proposed control is provably robust to saturation of the input and unmodeled linear actuator dynamics.

A contribution of this work is a fully distributed collision avoidance strategy that naturally arises from the robustness properties of the control and preserves the stability of the closed-loop system. 
Much of the distributed formation control literature do not consider collision avoidance (e.g., in the original BCB approach \cite{Lin2014, Lin2016}), and existing collision avoidance approaches are often centralized. 
Furthermore, an ad hoc augmentation of a distributed formation control strategy with collision avoidance, e.g., using potential functions, can lead to undesired behaviors or even instability (e.g., robots may drift or move in a limit cycle indefinitely).

We further present a portable and low-cost distributed robotic platform that consists of off-the-shelf components (see Fig.~\ref{fig:UTD}). 
This platform is used to validate our proposed formation control experimentally and can be used to test other multi-agent control strategies.
Since the platform is distributed, the number of robots used for an experiment is only limited by the available resources.
The code and technical implementation details related to this platform are made available online, and are accessible in the Supplementary Material section.  

In order to make the paper self-contained, this comprehensive work contains a summary of the relevant results derived in our previous papers \cite{Fathian2017, Fathian2018, Fathian2019, Fathian2018b} and subsequent extensions after the submission of this manuscript \cite{Fathian2019a, Lusk2020}.
Specifically, 
\cite{Fathian2017} studied the BCB control design under arbitrary switching sensing typologies,  
\cite{Fathian2018} presented the initial extension of the BCB control to agents with higher order linear dynamics,
\cite{Fathian2019} presented an augmentation of the BCB control to fix the formation scale with convergence guarantees, 
and \cite{Fathian2018b} presented the extension to agents with a kinematic unicycle model, particularly for fixed-wing aerial vehicles.
Contributions of this work include extension of the BCB control to agents with \textit{dynamic} unicycle and car models with convergence guarantees even in the presence of unmodeled linear actuator dynamics, and collision avoidance with stability guarantees. These contributions are accompanied by thorough simulation and experimental evaluations on an open-source robotic platform. 
Our latest extensions \cite{Fathian2019a, Lusk2020} expand the BCB control to 3D formations and leverage task assignment to mitigate gridlock scenarios that arise due to the distributed collision avoidance, respectively.

In summary, the main contributions of this paper are
\begin{itemize}
	\item A \textit{distributed}, \textit{provably convergent} and \textit{robust} formation control strategy for vehicles with a large variety of holonomic and nonholonomic dynamics, which \underline{\textit{eliminates}} the need for \textit{global position measurements}, \textit{common heading direction}, \textit{inter-agent communication}, and \textit{complete sensing graph} required in existing formation control literature.
	
	\item A fully \textit{distributed collision avoidance} algorithm naturally incorporated in the formation control strategy with \textit{stability guarantees}.
	
	\item A low-cost distributed robotic platform with off-the-shelf components for validation of formation control pipeline.
\end{itemize}

\subsection{Paper Organization}

The notation and assumptions used throughout the paper are introduced in Section \ref{sec:Notations}. 
In Section \ref{sec:SingleInt}, the control strategy for agents with 
single-integrator dynamics is introduced, the SDP gain design algorithm is presented, and robustness of the proposed approach to perturbations and saturated input is proven.
Gains designed for single-integrator agents are used in Section \ref{sec:HigherOrder} to extend the control to agents with higher order linear or linearizable holonomic dynamics such as quadrotors. 
In Sections \ref{sec:Unicycle} and \ref{sec:Car}, the control is further extended to agents with nonholonomic unicycle and car dynamics, where robustness to saturations in the input and unmodeled dynamics is shown. 
Additional topics such as collision avoidance, time-varying sensing topologies, scale of the formation, and extension to 3D case are discussed in Section \ref{sec:Extensions}. 
Lastly, in Sections \ref{sec:Simulations} and \ref{sec:Experiments} simulation and experimental results are presented to typify the proposed strategy.

\section{Notation and Assumptions}
\label{sec:Notations}

We consider a team of $n \in \bn$ agents with the inter-agent sensing topology described by an undirected graph $\mathcal{G} = (\mathcal{V},\mathcal{E})$, where $\mathcal{V} =\mathbb{N}_n := \{1,\, 2,\, \dots,\, n\}$ is the set of vertices, and $\mathcal{E} \subset \mathcal{V} \times \mathcal{V}$ is the set of edges. 
Each vertex of the graph represents an agent. An edge from vertex $i \in \mathcal{V} $ to $j \in \mathcal{V} $ indicates that agents $i$ and $j$ can measure the relative position of each other in their local coordinate frames. In such a case, agents $i$ and $j$ are called neighbors. The set of neighbors of agent $i$ is denoted by ${\mathcal{N}_i := \{j \in \mathcal{V}  \,|\, (i,j) \in \mathcal{E} \} }$. We denote by $\mathrm{eig}(A) \subset \bc$ the set of eigenvalues of matrix $A$.

Throughout this paper we assume that the desired formation and the sensing topology are such that achieving the formation is physically feasible. In particular, we assume that the sensing topology is 
undirected and universally rigid. This assumption is both necessary and sufficient \cite{Wang2013, Lin2016} for guaranteeing the existence of control gains that are computed from the proposed SDP approach.  
We further point out that by ``formation'' we imply a desired geometric shape up to a positive scale factor. To fix the scale of the formation to a desired value an augmented control is presented in Section \ref{sec:Extensions}.

\section{Formation Control for Single-integrator Dynamics} \label{sec:SingleInt}

In this section, we present the distributed formation control strategy introduced in \cite{Lin2014} for agents with single-integrator dynamics. We then propose a novel design approach for finding stabilizing control gains by formulating a convex optimization problem. The results of this section are a cornerstone for formation control of agents with more complicated dynamic models that are discussed in the subsequent sections.

\subsection{Control Strategy}
\label{sec:ControlDesign}

The single-integrator dynamics can be described as
%
\begin{gather} \label{eq:SingInt}
\dot{q}_i = u_i,
\end{gather}
%
where $q_i := [x_i,\, y_i]^\top \in \mathbb{R}^{2}$ is the coordinate of agent $i \in \mathbb{N}_n$ in a common global coordinate frame (unknown to the agent), and $u_i \in \br^2$ is the control law.
To bring the agents to a desired formation, the control law for each agent can be chosen as
%
\begin{gather} \label{eq:HolonomCtrl}
u_i := \sum_{j \in \mathcal{N}_i}{A_{ij} \, (q_j-q_i)},
\end{gather}
%
where $A_{ij} \in \mathbb{R}^{2\times2}$ are constant control gain matrices that will be designed later, and each has the form 
%
\begin{gather} \label{eq:Aij}
A_{ij} := \begin{bmatrix} 
a_{ij}& b_{ij} \\ 
-b_{ij} & a_{ij} 
\end{bmatrix}, \qquad a_{ij},\, b_{ij} \in \mathbb{R}.
\end{gather}
%
Thanks to the commutativity property of the $A_{ij}$ matrices,  the closed-loop dynamics with coordinates $q_i$ and $q_j$ expressed in agents' local coordinate frames is identical to the case that coordinates are expressed in a global coordinate frame (for more details see \cite{Fathian2017}). 
The geometric intuition behind the control strategy \eqref{eq:HolonomCtrl} is explained in the following example.

\begin{figure}
	\begin{center}
		\includegraphics[trim =95mm 55mm 80mm 65mm, clip, width=0.23\textwidth]{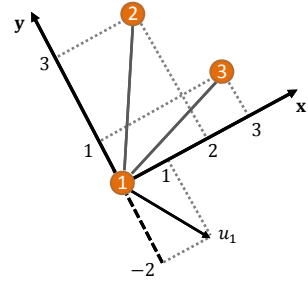}	
		\caption{Example of three agents with agents 2 and 3 neighbors of agent 1.}
		\label{fig:Example}
	\end{center}
\end{figure}

\begin{example}
Consider three agents in Fig.~\ref{fig:Example}, where agents 2 and 3 are neighbors of agent 1. Let $q_2 = [2,\, 3]^\top$ and $q_3 = [3,\, 1]^\top$ denote the position of neighbors in agent 1's local coordinate frame, and assume that control gains for agent 1 are given as
%
\begin{gather}
A_{12} =  \begin{bmatrix}
2 & -1  \\
1 & 2
\end{bmatrix}, \qquad
A_{13} = \begin{bmatrix}
-1 &  3 \\
-3 & -1
\end{bmatrix}.
\end{gather}
%
From \eqref{eq:HolonomCtrl}, the control vector for agent 1 is computed as
%
\begin{gather}
u_1 = A_{12} \, q_2 +  A_{13} \, q_3 = \begin{bmatrix}
1 \\
-2 
\end{bmatrix},
\end{gather}
%
which is shown in the figure and can be interpreted geometrically as follows. At any instance of time, agent 1 moves along the control vector with the speed equal to the vector's magnitude. Note that due to the special structure of gain matrices $A_{12},\, A_{13}$, they can be interpreted as scaled rotation matrices that rotate and scale  vectors connecting agent 1 to its neighbors.  One can see that this action is independent of agent 1's local coordinate frame position and orientation, hence, $q_1$ and $q_2$ can replaced by their coordinates in a global coordinate frame for analysis.	
\label{ex:Example}
\end{example}

Let $q := [q_1^\top,\, q_2^\top, \, \dots, q_n^\top]^\top \in \br^{2n}$ 
denote the aggregate state vector of all agents. Using this notation, the closed-loop dynamics under the control strategy \eqref{eq:HolonomCtrl} can be expressed as 
%
\begin{gather} \label{eq:SingIntClosed}
\dot{q} = A\, q,
\end{gather}
%
%
\begin{equation*} 
A = \begin{bmatrix}
-\sum_{j = 2}^{n} A_{1j} & A_{12} & \cdots & A_{1n} \\
A_{21} & -\sum_{\substack{{j = 1} \\ {j\neq2}}}^{n} A_{2j} & \cdots & A_{2n} \\
\vdots &                          & \ddots & \vdots \\
A_{n1} &       A_{n2}             & \cdots &  -\sum_{j = 1}^{n-1} A_{nj}   
\end{bmatrix} \in \mathbb{R}^{2n\times 2n},
\end{equation*}
%
where for $j \notin \mathcal{N}_i$ the $A_{ij}$ block is defined as a zero matrix. Note that the  $2\times 2$ diagonal blocks of $A$ are the negative sum of the rest of the blocks on the same row. Hence, $A$ has block Laplacian structure, and it follows that vectors
%
\begin{gather} \label{eq:ones}
\begin{aligned}
\mathbf{1} &:= [1,\, 0,\, 1,\, 0,\, \dots,\, 1,\, 0]^\top \in \br^{2n} \\
\bar{\mathbf{1}} &:=[0,\, 1,\, 0,\, 1,\, \dots,\, 0,\, 1]^\top \in \br^{2n}
\end{aligned}
\end{gather}
%
are in the kernel\footnote{If $A\in \mathbb{R}^{n\times n}$, the kernel or null space of $A$ is defined as ${ \mathrm{ker}(A) := \{ v\in \mathbb{R}^n \,|\, A \, v = 0\} }$.} of $A$.

Let $q^* \in \br^{2n}$ denote the coordinates of agents at the desired formation (the orientation, translation, and scale of the desired formation can be chosen arbitrarily). Further, let $\bar{q}^* \in \br^{2n}$ denote the coordinates of agents when the desired formation is rotated by $90$ degrees about the origin.  
The following theorem states the conditions that guarantee the convergence of agents to the desired formation.

\begin{theorem} \label{thm:PConvergence}
Consider agents with single-integrator dynamics \eqref{eq:SingInt} and control \eqref{eq:HolonomCtrl}. If the $A_{ij}$'s are chosen such that in \eqref{eq:SingIntClosed}
\begin{enumerate}[label=(\roman*)]
	\item $A$ has null vectors $\mathbf{1},\, \bar{\mathbf{1}}, \, q^*$ and $\bar{q}^*$,
	\item Other than the four zero eigenvalues associated with these null vectors, all eigenvalues of $A$ have negative real parts,
\end{enumerate}
then, agents globally converge to the desired formation.
\end{theorem}

\begin{proof}
The formal proof can be found in our previous work \cite[Thm. 1]{Fathian2017}, and is based on the observation  that if nonzero eigenvalues of matrix $A$ have negative real parts, all trajectories of the linear system $\dot{q} = A \, q$ exponentially converge to the kernel of $A$. The kernel of $A$ is nothing but all rotations, translations, and non-negative scale factors of the desired formation.
\end{proof}

Note that in Theorem \ref{thm:PConvergence} convergence to the desired formation implies that the formation is achieved up to a rotation and translation in the global coordinate frame, and a non-negative scale factor.
As we will discuss in Section \ref{sec:Extensions}, in applications where the  scale is important, the control can be augmented  to attain the desired scale. 
We should point out that null vectors $\mathbf{1}, \bar{\mathbf{1}}$ correspond to the case where all agents coincide, which can be interpreted as the desired formation achieved with the zero scale. It can be shown that the set of initial conditions that converge to this coinciding equilibrium is  measure zero. Notice that in practice, trajectories of agents cannot remain on a measure zero set (due to noise, disturbances, etc.), thus, coinciding agents are not of practical concern. 
%

\begin{remark} \label{rem:Rigidity}
The topological conditions that guarantee the existence of a symmetric matrix $A$ satisfying the conditions of Theorem~\ref{thm:PConvergence} are studied in \cite[Thm. 3.2]{Lin2016a}, which presents the necessary and sufficient condition\footnote{To be specific, the necessary and sufficient condition is for a generic desired formation. For certain desired formations matrix $A$  exists even when the graph is not universally rigid.} that the sensing graph is undirected and universally rigid. Throughout this paper, we assume that this condition is met.
\end{remark}

\phantomsection \label{p:b12} 
\begin{remark} \label{rem:MovingFrom}
Motion of the ensemble set of agents, during and after getting into formation, can be addressed in various ways.  We have previously employed leader-follower strategies and preassigned high-level control tasks \cite{Fathian2018b}.  Such a task may require global information and is not discussed further here.
\end{remark}

\subsection{Control Gain Design } \label{sec:Optimization}

Given a desired formation for agents with a universally rigid sensing topology, we present a novel algorithm to find control gain matrices that meet the conditions of Theorem \ref{thm:PConvergence}. 
Let $N := [q^*,\, \bar{q}^*,\, \mathbf{1},\, \bar{\mathbf{1}}] \in \br^{2n\times 4}$ be the set of bases for the kernel of $A$, where $\mathbf{1},\, \bar{\mathbf{1}}$ are given in \eqref{eq:ones},  $q^* \in \br^{2n}$ is the coordinates of agents at the desired formation, and $\bar{q}^* \br^{2n}$ is the $90^\circ$ rotated coordinates about the origin. 
Let $U \, S \, V^\top = N$ be the (full) singular value decomposition (SVD) of $N$, where
%
\begin{equation} \label{eq:Q}
U = [\bar{Q}, ~ Q]  ~\in \mathbb{R}^{2n\times 2n},
\end{equation}
%
with $Q\in \mathbb{R}^{2n\times (2n-4)}$ defined as the last $2n-4$ columns of $U$.

\begin{lemma} \label{lem:Lbar}
Using $Q$ in \eqref{eq:Q}, define 
\begin{equation}
\bar{A} := Q^\top  A \, Q ~\in \mathbb{R}^{(2n-4)\times (2n-4)}.
\label{eq:Lbar}
\end{equation}
%
Matrices $A$ and $\bar{A}$ have the same set of nonzero eigenvalues. 
\end{lemma}

Proof of Lemma \ref{lem:Lbar} follows by observing that $U$ is an orthogonal matrix, and $\mathrm{range}(\bar{Q}) = \mathrm{range}(N)$. Therefore $\bar{A}$ is the restriction of $A$ onto the orthogonal complement of $\mathrm{range}(N)$, which removes the zero eigenvalues of $A$.

For an undirected sensing topology, by imposing the constraints $a_{ij} = a_{ji}$, $b_{ij} = - b_{ji}$ in \eqref{eq:Aij} matrix $A$ can be designed to be symmetric.
Note that from Remark~\ref{rem:Rigidity} existence of such matrix is guaranteed.
In this case, $\bar{A}$ is symmetric, and its eigenvalues are real and can be ordered. Hence, $A$ can be computed by solving the optimization problem
%
\begin{align} \label{eq:OptimCVX}
A ~=~ \underset{a_{ij},\, b_{ij}}{\mathrm{argmax}}~~~& \quad \lambda_1 (-\bar{A})   \\
\text{subject to}& \quad A \, N = 0  \nonumber \\
& \quad \mathrm{trace}(A) = \text{constant}. \nonumber  
\end{align}
%
where $\lambda_1(\cdot)$ denote the smallest eigenvalue of a matrix and the last constraint ensures that the solution remains bounded. 
Note that \eqref{eq:OptimCVX} is a concave maximization problem \cite{Boyd2006}, and can be formulated as the SDP problem
%
\begin{align} \label{eq:OptimSDP}
A ~=~ \underset{a_{ij},\, b_{ij}, \, \gamma}{\mathrm{argmax}}~~~& \quad \gamma   \\
\text{subject to}& \quad \bar{A} + \gamma \, I \preceq 0 \nonumber \\
& \quad A \, N = 0 \nonumber  \\
& \quad \mathrm{trace}(A) = \text{constant}. \nonumber  
\end{align}
%
where the first constraint is a linear matrix inequality. 
The proposed approach for finding stabilizing gain matrix $A$ is summarized in Algorithm~\ref{alg:GainDesign}.

Several effective algorithms for solving SDPs are developed in recent years \cite{Tuetuencue2003} that can be used to solve problem \eqref{eq:OptimCVX}.  CVX \cite{cvx} is well-suited to solve \eqref{eq:OptimCVX} when the number of agents is less than 50, and it features a relatively simple interface. 
For scenarios with larger number of agents, customized and more computationally efficient solvers can be leveraged to obtain an answer. 
In  \cite{Lusk2020}, we presented an ADMM-based customized solver for \eqref{eq:OptimCVX}. 
Table~\ref{tbl:gaintimes} shows the time required to solve \eqref{eq:OptimCVX} for a random sensing topology in MATLAB using an Intel Core i7-7700K with 16GB RAM.
As it can be seen, by using the ADMM-based solver gains for formations of 100 agents can be computed in less than 11 seconds.

\begin{table}[t!]	
	\centering
	\caption{Execution time of the CVX solver used for \eqref{eq:OptimCVX} vs. our customized ADMM solver in \cite{Lusk2020} for obtaining 2D formation gains for different number of agents. Reported times are in seconds and rounded to two decimals.}
	\begin{tabularx}{\columnwidth}{@{} l rrrrr @{}} 
		\toprule
		Algorithm  &  \multicolumn{5}{c @{}}{Number of Agents}\\ 
		\cmidrule(l){2-6}
		& {5} & {20} & {50} & {100} &  {200}  \\
		\midrule
		CVX-SDP time         & 0.49 & 9.71 & 6042.96 &  \texttt{OOM} & \texttt{OOM}  \\
		ADMM time & 0.01 & 0.05 & 1.06 & 10.91 & 123.76  \\
		\bottomrule
	\end{tabularx}
	\begin{flushleft}
	{\scriptsize \texttt{OOM}: Out of memory}
	\end{flushleft}
	\label{tbl:gaintimes}
\end{table}

\begin{algorithm}[t] 
%
\DontPrintSemicolon
\SetKwData{Left}{left}\SetKwData{This}{this}\SetKwData{Up}{up}
\SetKwFunction{Union}{Union}\SetKwFunction{FindCompress}{FindCompress}
\SetKwInOut{Input}{input}\SetKwInOut{Output}{output}

\SetKwInput{StepA}{step 1}
\SetKwInput{StepB}{step 2}
\SetKwInput{StepC}{step 3}
\SetKwInput{StepD}{step 4}
\SetKwInput{Notation}{notation}

\caption{Formation control gain design.}

\Input{Desired formation coordinates $q^*$.}
\Output{Gain matrix $A$.}

\BlankLine

\StepA{Let $N := [q^*,\, \bar{q}^*,\, \mathbf{1},\, \bar{\mathbf{1}}]$.} 
\StepB{Compute SVD of $N = U\, S \, V^\top$.}
\StepC{Define $Q$ as the last $2n- 4$ columns of $U$.}
\StepD{Solve \eqref{eq:OptimCVX} using a SDP solver.}
%
%
\label{alg:GainDesign}
\end{algorithm}

It is important to note the distinction between the \textit{design} phase and \textit{implementation} in our approach.
Designing the control gains by Algorithm~\ref{alg:GainDesign} is a \textit{centralized} paradigm (which requires the knowledge of the sensing topology).
These gains are transmitted from the base station to agents to be used during the mission. The implementation of our approach is \textit{distributed}, where agents use the prescribed gains to achieve the desired formation without a need for communication and using only relative/local position measurements.
Distributed optimization techniques can solve \eqref{eq:OptimSDP} without relying on the complete knowledge of the sensing topology. An example of such distributed design can be found in \cite{Lin2016a}. However, these techniques require inter-agent communication, which we avoid in this work.

\subsection{Robustness to Perturbations} \label{sec:Robustness}

An important characteristic of the proposed design approach is that the gains found via \eqref{eq:OptimCVX} lead to significant robustness to perturbations.   
For instance, noise and disturbances can cause an agent to move in a direction that is different from the desired control vector. 
The following theorem shows that by using the gains computed from \eqref{eq:OptimCVX} positive scaling and rotation of the control vectors (up to $\pm 90^\circ$) does not affect the convergence.

\begin{theorem} \label{thm:Robustness}
Given control gain matrix $A$ designed from \eqref{eq:OptimCVX}, let $R_i \in \mathrm{SO}(2)$ denote a rotation matrix of $\alpha_i$ radians, and $c_i \in \br$ be a scalar. If $\alpha_i \in [-\frac{\pi}{2}+\epsilon,\, \frac{\pi}{2}-\epsilon]$ for an arbitrary small $\epsilon >0$, and $c_i > 0$, under the perturbed control 	
%
\begin{equation} \label{eq:PurturbedControl}
u_i := c_i \, R_i \, \sum_{j \in \mathcal{N}_i}{A_{ij} \, (q_j-q_i)}
\end{equation}
%
single-integrator agents achieve the desired formation.	 
\end{theorem}

We first present and prove the following lemma that is used in the proof of Theorem~\ref{thm:Robustness}.

\begin{lemma} \label{lem:PosDefRotation}
Let $R \in \mathrm{SO}(2)$ represent a rotation of $\alpha \in [-\pi,\, \pi)$ radians. If $|\alpha|<\frac{\pi}{2}$, then $R + R^\top$ is positive definite.	
\end{lemma}

\begin{proof}
Matrix $R \in \mathrm{SO}(2)$ can be represented as $R = \begin{bmatrix}
c & -s \\
s & c \\
\end{bmatrix}$,
where $c, \, s$ are shorthand notations for $\cos(\alpha),\, \sin(\alpha)$, respectively. Hence, $R + R^\top = \begin{bmatrix}
2c & 0 \\
0 & 2c \\
\end{bmatrix}$,
which since for $|\alpha|<\frac{\pi}{2}$ we have $c > 0$, and matrix $R + R^\top$ is positive definite. 
\end{proof}

We now present the proof of Theorem~\ref{thm:Robustness}.

\begin{proof}
Under the perturbed control \eqref{eq:PurturbedControl}, the aggregate dynamics can be represented by 
%
\begin{equation} \label{eq:PurturbedDynamic}
\dot{q} := P \, A \, q
\end{equation}
%
where $P := \mathrm{diag}(c_1 R_1,\, c_2 R_2,\, \dots,\, c_n R_n) \in \br^{2n \times 2n}$ is a block diagonal matrix that contains the perturbation terms. 
Consider the Lyapunov function candidate $V := - q^\top A\, q$. Note that $V$ is positive semidefinite since by design $A$ is negative semidefinite, and $V = 0$ if and only if $q \in \mathrm{ker}(A)$. Noting that $A^\top = A$, derivative of $V$ along the trajectories of \eqref{eq:PurturbedDynamic} is
%
\begin{gather} \label{eq:}
\dot{V} = -\dot{q}^\top A \, q - q^\top A \, \dot{q}  
= - q^\top A \left(  P^\top + P  \right)\, A \, q.
\end{gather}
%
Matrix $P^\top + P$ is block diagonal and each diagonal block is given by $c_i \, (R_i^\top + R_i) \in \br^{2\times 2}$. From Lemma~\ref{lem:PosDefRotation}, we have that if $|\alpha_i| < \frac{\pi}{2}$ and $c_i > 0$ for all $i \in \{1,\, \dots,\, n\}$, then all diagonal blocks are positive definite. This implies that ${P^\top + P}$ is positive definite, and consequently $\dot{V} < 0$ for all $q \notin \mathrm{ker}(A)$. From the Lyapunov stability theory and LaSalle's invariance principle \cite{Khalil1996} it then follows that all trajectories of \eqref{eq:PurturbedDynamic} converge to the invariant set $q \in \mathrm{ker}(A)$, which shows that the desired formation is achieved.  
\end{proof}

\subsection{Robustness to Saturated Input}

In practice, the velocity of an agent cannot take arbitrary large values. Thus, any large control input will be saturated by a maximum feasible/allowed speed. This, however, does not affect convergence of agents to the desired formation. 

\begin{theorem} \label{thm:SingIntSaturated}
Consider single-integrator agents with dynamics \eqref{eq:SingInt} and assume that $u_{\max} > 0$ is a real positive scalar. If $u_i$ is saturated such that $|u_i| \leq u_{\max}$, then under the control \eqref{eq:HolonomCtrl} the desired formation is achieved globally.  
\end{theorem}

\begin{proof}

We first discuss the following Lemma and  Corollary:

\begin{lemma} \label{lem:SwitchedSys}
	\cite[Sec. 2.1.2]{Liberzon2012}
	Consider the family of switched systems $\dot{x} = f_i(x)$, with $i = 1,2, \dots, N$. Let $V: \br^n \rightarrow \br$ be a positive definite, continuously differentiable, and radially unbounded function. If $\frac{\partial V}{\partial x} f_i(x) < 0,\, \forall x \neq 0,\, \forall i$, then the switched system is globally uniformly asymptotically stable.
\end{lemma}

\begin{corollary} \label{cor:SwitchedSysExtension}
	Lemma \ref{lem:SwitchedSys} can be extended to a positive semidefinite $V$  with the zero set of $Z := \{ x \in \br^n \, : \, V(x) = 0 \}$ . In this case, if $\frac{\partial V}{\partial x} f_i(x) < 0,\, \forall x \notin  Z,\, \forall i$, then all trajectories globally uniformly asymptotically converge to $Z$.
\end{corollary}

To model the input saturation we can define the diagonal matrix $S \in \br^{n \times n}$ with diagonal elements
%
\begin{gather} \label{eq:S}
(S)_{ii} = \begin{cases}
1 & \text{if} ~ |u_i| \leq u_{\max}  \\
\frac{u_{\max}}{|u_i|} & \text{if} ~ |u_i| > u_{\max}.
\end{cases}
\end{gather}
%
As illustrated in Fig.~\ref{fig:Saturate}, diagonal elements of $S$ can be considered as functions that saturate any large input to the maximum value $u_{\max}$.
The closed-loop dynamics under the saturated input can be expressed in the vector form via
%
\begin{gather} \label{eq:SingIntSat}
\dot{q} =  S\,  A \, q 
\end{gather}
%
System \eqref{eq:SingIntSat} should be understood as a family of switched dynamical systems, for which the solution is well-defined in the Filippov sense (see Chapter 2 in \cite{Liberzon2012} for more details). To show that this system is uniformly stable, we consider 
%
\begin{gather} \label{eq:}
V := -\frac{1}{2} q^\top A \, q \geq 0
\end{gather}
%
as a common Lyapunov function candidate for all systems. Note that since $A$ is negative semidefinite, $V$ is a positive semidefinite scalar valued  function.
Time derivative of $V$ along the trajectory of \eqref{eq:SingIntSat} is
%
\begin{align} \label{eq:}	
\dot{V} &= - q^\top A \, \dot{q} \nonumber \\
&= - q^\top A \, S\, A \, q \nonumber  \\
&= - (S^{\frac{1}{2}} \,  A \, q)^\top (S^{\frac{1}{2}}\, A \, q) = - \| S^{\frac{1}{2}} \, A \, q \|^2 ~ \leq \, 0,
\end{align}
%
where $S^{\frac{1}{2}}$ is the diagonal matrix with elements given by the square root of diagonal entries of $S$. Note that all diagonal elements of $S$ are strictly positive, hence $S^{\frac{1}{2}}$ is well-defined. 
Since $V$ is a positive semidefinite, continuously differentiable, and radially unbounded function, from Lemma \ref{lem:SwitchedSys}, Corollary \ref{cor:SwitchedSysExtension}, and LaSalle's invariance principle it follows that all trajectories of \eqref{eq:SingIntSat} converge to  the zero set of $V$, which is the kernel of $A$. Thus, the desired formation is achieved. 
\end{proof}

\begin{figure}
	\begin{center}
		\includegraphics[trim = 60mm 48mm 70mm 48mm, clip, width=0.26\textwidth] {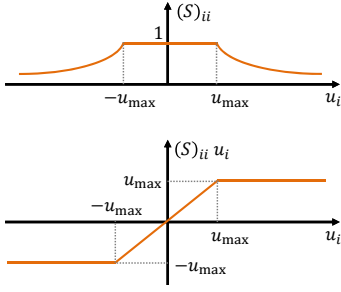}
		\caption{Top: The $i$-th diagonal entry of matrix $S$. Bottom: The effect of saturation on the control.}	
		\label{fig:Saturate}			
	\end{center}
\end{figure}

\begin{remark} \label{rem:SteadyStateError}
To reject steady state errors, the control law \eqref{eq:HolonomCtrl} can be augmented by an integrator term as 
%
\begin{gather} \label{eq:HolonomCtrlInteg}
u_i := k_0 \, \sum_{j \in \mathcal{N}_i}{A_{ij} \, (q_j-q_i)} + k_1 \, \int_{0}^{t}{\sum_{j \in \mathcal{N}_i}{A_{ij} \, (q_j-q_i)} \, d\tau},
\end{gather}
%
where $k_0, \, k_1 \in \mathbb{R}$ are scalar control gains. It can be shown that if $k_0, \, k_1 >0 $, this augmented control rejects constant input/output disturbances (see \cite[Sec. III-D]{Fathian2018} for more details).
\end{remark}

\phantomsection \label{p:b6} 
\begin{remark} \label{rem:}
The robustness properties of control \eqref{eq:HolonomCtrl}, such as robustness to positive scaling and rotations up to $\pm 90^\circ$, are similar to the properties of the first-order consensus methods. This originates from the structure of the Lyapunov analysis that is similar in both approaches. However, consensus-based methods \cite{Lawton2003, Ren2007} require alignment of agents' local coordinate frames, whereas the formation control strategy studied in this work does not have this constraint.
\end{remark}

\section{Formation Control for Agents with Higher-order Dynamics} \label{sec:HigherOrder}

In this section, we extend the single-integrator control strategy to agents with higher-order dynamics. We show how the control gains designed for single-integrator agents in Section \ref{sec:Optimization} can be used directly to control higher-order agents \textit{without having to find a new control strategy or redesign the gains by solving a new optimization problem.}  This means the same formation can be regulated for any  type of vehicle using the same gains.
We assume that the aggregate higher-order dynamics of all agents can be expressed in the controllable canonical form
%
\begin{gather} \label{eq:HigherOrderModel}
\begin{bmatrix}
\dot{q} 	  \\
\dot{q}^{(1)} \\
\vdots  	  \\
\dot{q}^{(m-1)} \\ 
\dot{q}^{(m)}   
\end{bmatrix}
= \begin{bmatrix}
0 & I & 0 & \cdots & 0  \\
0 & 0 & I &        & 0  \\
\vdots & & & \ddots & \vdots \\ 
0 & 0 & 0 &        & I \\
0 & 0 & 0 & \cdots & 0 
\end{bmatrix}
\begin{bmatrix}
q		 \\
q^{(1)}  \\
\vdots   \\
q^{(m-1)} \\ 
q^{(m)}   
\end{bmatrix} +
\begin{bmatrix}
0	  \\
0 \\
\vdots  	  \\
0 \\ 
I   
\end{bmatrix} \, u,
\end{gather}
%
where $q \in \br^{2n}$ is the aggregate position vector of all agents,  $q^{(j)} \in \br^{2n}$ denotes the $j$'th derivative of $q$, and $I \in \br^{n\times n}$ is the identity matrix.
Although at first sight \eqref{eq:HigherOrderModel} may seem restrictive, in fact, it encompasses a large class of agents. This is because by coordinate transformation techniques such as feedback linearization, or approximation techniques such as linearization and gain scheduling, dynamics of many systems can be expressed as \eqref{eq:HigherOrderModel}.

Given the gain matrix $A$ designed for agents with the single-integrator model, the control for agents with dynamics \eqref{eq:HigherOrderModel} can be chosen as
%
\begin{equation} \label{eq:HigherOrderControl}
u =  k_0 \, A \, q +  k_1 \, A \, q^{(1)} + \dots +  k_m \, A \, q^{(m)},
\end{equation}
%
where $k_0,\, k_1,\, \dots,\, k_m \in \br$ are scalar control gains, and $u := [u_1^\top,\, u_2^\top,\, \dots, u_n^\top]^\top \in \br^{2n}$ denote the aggregate control vector. Note that \eqref{eq:HigherOrderControl} can be implemented locally using only the relative measurements (due to the special structure of $A$). Under this control, the closed-loop dynamics is given by
%
\begin{gather} \label{eq:HigherOrderClosed}
\begin{bmatrix}
\dot{q} 	  \\
\dot{q}^{(1)} \\
\vdots  	  \\
\dot{q}^{(m \mbox{-} 1)} \\ 
\dot{q}^{(m)}   
\end{bmatrix}
= \underbrace{\begin{bmatrix}
	0 & I & 0 & \cdots & 0  \\
	0 & 0 & I &        & 0  \\
	\vdots & & & \ddots & \vdots \\ 
	0 & 0 & 0 &        & I \\
	k_0 A& k_1 A & k_2 A & \cdots & k_m A  
	\end{bmatrix}}_{E}
\begin{bmatrix}
q		 \\
q^{(1)}  \\
\vdots   \\
q^{(m \mbox{-} 1)} \\ 
q^{(m)}   
\end{bmatrix}.
\end{gather}
%

\begin{theorem} \label{thm:HigherOrderConvergence}
If for all nonzero $\mu \in \mathrm{eig}(A)$ roots of the polynomial equation
%
\begin{gather} \label{eq:DegreeM}
\lambda^{m+1} - k_m \, \mu \, \lambda^{m} -\dots - k_1 \, \mu \, \lambda - k_0 \, \mu = 0
\end{gather}
%
have negative real parts, then under control \eqref{eq:HigherOrderControl}, agents with dynamics \eqref{eq:HigherOrderModel} globally converge to the desired formation.
\end{theorem}

Before we present the proof of Theorem~\ref{thm:HigherOrderConvergence}, we present and prove the following Lemma.

\begin{lemma} \label{lem:PolyEigens}
Let $p(\cdot)$ be a given polynomial. If $\mu$ is an eigenvalue
of matrix $A$ with $v$ as the associated eigenvector, then $p(\mu)$ is an eigenvalue of the matrix $p(A)$ with $v$ as the associated eigenvector.	
\end{lemma}

\begin{proof}
Let $p(\cdot)$ be a polynomial of degree $k$, and consider 
%
\begin{gather} \label{eq:polyA}
p(A)\, v = a_k A^k v + a_{k-1} A^{k-1} v + \dots + a_1 A\, v + a_0 \, v,
\end{gather}
%
where $a_j$'s,  $j = 0, \dots, k$,  are coefficients of the polynomial. Since $v$ is an eigenvector, we have $ A^j v = A^{j-1} (A\,v) = A^{j-1} (\mu \, v) = \mu (A^{j-1} v) = \dots = \mu^j v$. Thus, from \eqref{eq:polyA} we get
%
\begin{gather*} \label{eq:}
p(A)\, v = (a_k \mu^k  + a_{k-1} \mu^{k-1}  + \dots + a_1 \mu\,  + a_0) \, v = p(\mu)\, v,
\end{gather*}
%
which concludes the proof.
\end{proof}

We now present the proof of Theorem~\ref{thm:HigherOrderConvergence}.

\begin{proof}
The closed-loop state matrix $E$, defined in \eqref{eq:HigherOrderClosed}, is in the (block) controllable canonical form. From this observation and Lemma \ref{lem:PolyEigens}, the characteristic equation of $E$ is given by
%
\begin{align} \label{eq:}
& \det( \lambda^{m+1}\, I - k_m \, \lambda^{m} \, A -\dots - k_1 \, \lambda\, A - k_0 \, A )  \nonumber \\
& =  \prod_{\mu \in \mathrm{eig}(A)}{( \lambda^{m+1}\, I - k_m \, \lambda^{m} \, \mu -\dots - k_1 \, \lambda\, \mu - k_0 \, \mu )} = 0,
\end{align}
%
which from the assumption of the theorem implies that the nonzero eigenvalues of $E$ have negative real parts.
\end{proof}

To find gains $k_0,\, k_1,\, \dots, k_m$ that satisfy the condition of Theorem \ref{thm:HigherOrderConvergence}, the Routh-Hurwitz criterion can be used.

\begin{remark} \label{rem:HigherOrder}
In the above analysis, the control can alternatively be chosen as
%
\begin{gather} \label{eq:HigherOrderControl2}
u =  k_0 \, A \, q +  k_1  \, q^{(1)} + \dots +  k_m  \, q^{(m)}.
\end{gather}
%
In this case, agents do not need measurements of states $q^{(1)}, \, \dots,\, q^{(m)}$ for their neighbors (since $A$ is replaced by the identity matrix). Note that \eqref{eq:HigherOrderControl2} can also be implemented using only the local relative measurements. 
%
%
\end{remark}

\phantomsection \label{p:b11}
\begin{remark} \label{rem:Sensing}
There are several methods that can be used to determine the relative position of agents; e.g., our previous work \cite{Fathian2018a} used vision sensors. LIDAR could be used to augment vision for accurate displacement measurements. In what was discussed above, relative velocity, acceleration, etc., can be computed by taking time derivatives of the measured relative position and using appropriate filtering when measurements are noisy. 
\end{remark}

\begin{example}\textbf{(Quadrotor dynamics)\\} \label{ex:Quadrotor}
Quadrotor dynamics can be described as \cite{Sabatino2015}
%
\begin{subequations} \label{eq:QuadNonlinDynam}
\begin{align} 
\begin{bmatrix}
\ddot{x} \\
\ddot{y} \\
\ddot{z} \\
\end{bmatrix} &= 
R \, 
\begin{bmatrix}
0 \\
0 \\
u^a \\
\end{bmatrix} -
\begin{bmatrix}
0 \\
0 \\
g \\
\end{bmatrix},
\\
\begin{bmatrix}
\dot{\varphi} \\
\dot{\theta} \\
\dot{\psi} \\
\end{bmatrix} &= T \,
\,
\begin{bmatrix}
\omega_x \\
\omega_y \\
\omega_z \\
\end{bmatrix},
\\
\begin{bmatrix}
\dot{\omega}_x \\
\dot{\omega}_y \\
\dot{\omega}_z \\
\end{bmatrix} &= 
J^{-1} \,
\begin{bmatrix}
u^x \\
u^y \\
u^z \\
\end{bmatrix} - 
J^{-1} \, \left(
\begin{bmatrix}
\omega_x \\
\omega_y \\
\omega_z \\
\end{bmatrix} \, \times J \, 
\begin{bmatrix}
\omega_x \\
\omega_y \\
\omega_z \\
\end{bmatrix}
 \right),
\end{align}
\end{subequations}
%
where, as illustrated in Fig.~\ref{fig:Quad}, $x,\, y,\, z \in \br$ are coordinates of the quadrotor's center of mass in the world frame,
$\varphi, \, \theta, \, \psi$ are roll, pitch, yaw angles that describe the orientation of the quadrotor body frame in the world frame, 
$\omega_x,\, \omega_y,\, \omega_z$ are the angular body rates about  associated body axes, 
$g$ is the gravitational constant, 
$u^a$ is a mass-normalized thrust input,
and $u^x, u^y, u^z$ are moment inputs applied to the airframe about corresponding body axes. 
Further,  $J \in \br^{3\times 3}$ is the mass moment of inertia matrix,
$R \in \mathrm{SO}(3)$ is the rotation matrix parameterized in terms of $z$-$x$-$y$ Euler angles as
%
\begin{gather} \label{eq:R_Euler} 
R := \begin{bmatrix}
c_\psi \, c_\theta - s_\varphi \, s_\psi \, s_\theta & - c_\varphi \, s_\psi & c_\psi \, s_\theta + c_\theta \, s_\varphi \, s_\psi   \\
c_\theta \, s_\psi + c_\psi \, s_\varphi \, s_\theta & c_\varphi \, c_\psi & s_\psi \, s_\theta - c_\psi \, c_\theta \, s_\varphi \\
-c_\varphi \, s_\theta & s_\varphi & c_\varphi \, c_\theta \\
\end{bmatrix},
\end{gather}
%
where $c,\, s$ are respectively shorthand notations for $\cos(\cdot), \, \sin(\cdot)$ functions, and
%
\begin{gather} \label{eq:Transform} 
T := \frac{1}{c_\varphi} \,
\begin{bmatrix}
c_\varphi \, c_\theta & 0 & -c_\theta  \\
s_\varphi \, s_\theta & c_\varphi & -c_\theta \, s_\varphi \\
-s_\theta & 0 & c_\theta \\
\end{bmatrix} \in \br^{3\times 3}
\end{gather}
%
is the transformation matrix that relates the roll, pitch, yaw derivatives to the angular velocities in the body frame.

\begin{figure}
	\begin{center}
		\includegraphics[trim =85mm 72mm 95mm 75mm, clip, width=0.25\textwidth]{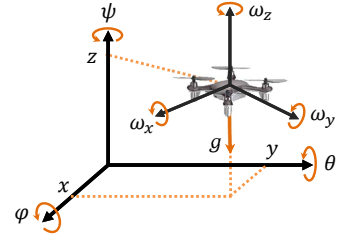}	
		\caption{Illustration of a quadrotor's body frame in the world frame.}
		\label{fig:Quad}
	\end{center}
\end{figure}

Linearizing dynamics \eqref{eq:QuadNonlinDynam} about the hover point $x = y = z = \dot{x} = \dot{y} = \dot{z} = 0$, $\omega_x = \omega_y = \omega_z = 0$, $u^x = u^y = u^z = 0$, and $u^a = g$ gives the quadrotor linearized dynamics 
%
\begin{gather} \label{eq:QuadLinDynamOrig}
\begin{matrix}
\delta \ddot{x} = ~g\, \delta \theta ~\,& \delta \ddot{\theta} = u^y \\
\delta \ddot{y} = -g\, \delta \varphi & \delta \ddot{\varphi} = u^x  \\
\delta \ddot{z} = ~u^a ~~~~& \delta \ddot{\psi} = u^z  \\
\end{matrix}
\end{gather}
%
where $\delta$ represents a small displacement about the equilibrium/linearization point.
Since we are interested in 2D formations, we only consider the lateral dynamics along the $x$-$y$ axes, and separately control the quadrotor's altitude by setting $u^a = \frac{g}{c_\varphi \, c_\theta}$ to stabilize it at a constant altitude.

To represent  the dynamics in the canonical form \eqref{eq:HigherOrderModel}, we define 
%
\begin{gather*} \label{eq:}
\delta \bar{\theta}_i := g\, \delta \theta_i, ~\quad \delta \bar{\varphi}_i := -g\, \delta \varphi_i, ~\quad \bar{u}^y_i := g\, u^y_i, ~\quad  \bar{u}^x_i := -g\, u^x_i ,
\end{gather*}
%
where subscript $i$ is used to distinguish agents.  Using this notation,  \eqref{eq:QuadLinDynamOrig} can be described in the vector form as
%
\begin{gather} \label{eq:QuadLinDynam}
\dot{p}_i = 
\begin{bmatrix}
0 & I & 0 & 0   \\
0 & 0 & I & 0   \\
0 & 0 & 0 & I   \\ 
0 & 0 & 0 & 0   \\ 
\end{bmatrix}
p_i +
\begin{bmatrix}
0   \\
0   \\
0   \\ 
I   \\ 
\end{bmatrix}
u_i
\end{gather}
%
where 
%
\begin{gather} \label{eq:QuadState}
\begin{aligned}
p_i := &{} [\delta x_i,\, \delta y_i,\, \delta \dot{x}_i,\, \delta \dot{y}_i,\, \delta \bar{\theta}_i,\, \delta \bar{\varphi}_i,\,  \delta \dot{\bar{\theta}}_i,\, \delta \dot{\bar{\varphi}}_i ]^\top ,\\
u_i := &{}[\bar{u}_i^y, \, \bar{u}_i^x]^\top ,
\end{aligned}
\end{gather}
%
are respectively the state and control vectors, and $I \in \br^{2\times 2}$ is the identity matrix. 
Note that by defining the aggregate position vector as $q = [\delta x_1,\, \delta y_1,\, \dots,\, \delta x_n,\, \delta y_n]^\top$, dynamics of agents can be expressed in the form \eqref{eq:HigherOrderModel}.  This model will be used in the Simulations section to achieve a desired formation.
\end{example}

\section{Formation Control for Agents with Unicycle Dynamics} \label{sec:Unicycle}

Motion profile of many vehicles, e.g., differential drive robots or fixed-wing aerial vehicles, can be described via the unicycle model.
In this section, we introduce the unicycle model and propose a formation control strategy to achieve the desired formation using the control gains that were designed for single-integrator agents. 
We then show that the desired formation is achieved even if the input is saturated, and the control strategy is robust to unknown dynamics that are not considered in the kinematic unicycle model.  
We assume henceforth that a symmetric negative semidefinite gain matrix $A$ is designed for the desired formation by solving the optimization problem \eqref{eq:OptimCVX}.

\subsection{Unicycle Dynamics} \label{sec:UnicycleDynam}

\phantomsection \label{p:b8} 
Consider a unicycle agent located at position $[x_i,\, y_i]^\top \in \br^2$ in a global coordinate frame (unknown to the agent), and assume that the unicycle's heading direction makes angle $\theta_i \in [0,\, 2\pi)$ with the $x$-axis of the global coordinate frame. This scenario is illustrated in Fig.~\ref{fig:control}. 
The unicycle dynamics can be described in the global coordinate frame by
%
\begin{gather} \label{eq:kinemUnicycle}
\begin{aligned}
\dot{x}_i &= v_i \, \cos{(\theta_i)} \\
\dot{y}_i &= v_i \, \sin{(\theta_i)} \\
\dot{\theta}_i &= \omega_i \\
\end{aligned}
\end{gather}
%
where scalars $v_i,\, \omega_i \in \mathbb{R}$ are respectively the linear and angular velocities of the agent.
In the unicycle kinematic model, it is assumed that $v_i$ and $\omega_i$ are control variables and can be changed instantaneously. 

In the global coordinate frame, the unit norm heading vector of the unicycle, $h_i \in \mathbb{R}^2$, and its perpendicular vector $h_i^\perp \in \mathbb{R}^2$, are given by
%
\begin{gather} \label{eq:hdef}
h_i := \begin{bmatrix}
\cos{(\theta_i)} \\
\sin{(\theta_i)}
\end{bmatrix},\qquad
h_i^\perp := \begin{bmatrix}
-\sin{(\theta_i)} \\
\cos{(\theta_i)}
\end{bmatrix}.
\end{gather}
%
Seeing that $\dot{h}_i = h_i^\perp \, \dot{\theta}_i$, \eqref{eq:kinemUnicycle} can be equivalently described by
%
\begin{gather} \label{eq:kinemUnicycle2}
\begin{aligned}
\dot{q}_i &=  h_i \, v_i \\
\dot{h}_i &= h_i^\perp \, \omega_i .
\end{aligned}
\end{gather}
%
%
Let $q := [q_1^\top,\, q_2^\top,\, \dots, \, q_n^\top]^\top \in \br^{2n}$ be the aggregate position vector of all agents, and similarly let $h \in \br^{2n}$, $v \in \br^{n}$, $\omega \in \br^{n}$ be the aggregate heading, linear velocity, and angular velocity vectors, respectively. 
Using this notation, the motion of all agents can be collectively expressed as
%
\begin{gather} \label{eq:KinemDynamVec}
\begin{aligned}
\dot{q} &= H \, v \\
\dot{h} &= H^\perp  \omega.
\end{aligned}
\end{gather}
%
where matrices ${H,\,  H^\perp \in \mathbb{R}^{2n\times n}}$ are defined as
%
\begin{gather} \label{eq:H}
H := \begin{bmatrix}
h_1 & 0 & \cdots & 0 \\
0  & h_2 &       & 0 \\
\vdots  &  & \ddots & \vdots \\
0  &  0 &  \cdots  & h_n  
\end{bmatrix}, \quad
H^\perp := \begin{bmatrix}
h_1^\perp & 0 & \cdots & 0 \\
0  & h_2^\perp &       & 0 \\
\vdots  &  & \ddots & \vdots \\
0  &  0 &  \cdots  & h_n^\perp  
\end{bmatrix}.
\end{gather}
%

\begin{figure} 
	\begin{center}	
		\includegraphics[trim = 90mm 75mm 80mm 60mm, clip, width=0.26\textwidth] {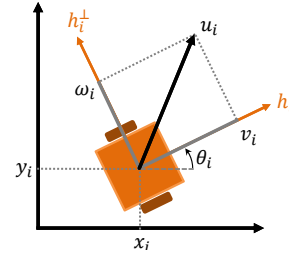}
		\caption{An agent with unicycle dynamics at position $(x_i,y_i)$ in the global coordinate frame. The agent's heading is denoted by $h_i$, and makes the angle $\theta_i$ with the global coordinate frame's $x$-axis. Scalars $v_i$ and $\omega_i$ are defined as the length of the control vector $u_i$ projected on $h_i$ and $h_i^\perp$, respectively.}	
		\label{fig:control}			
	\end{center}
\end{figure}

\subsection{Control Strategy}

Consider a team of $n$ unicycle agents with dynamics \eqref{eq:kinemUnicycle2}. We seek to assign controls $v_i$ and $\omega_i$ such that agents autonomously achieve a desired formation.  
Let $A \in \br^{2n \times 2n}$ be a symmetric gain matrix designed in Section \ref{sec:Optimization} for agents with single-integrator model to achieve the desired formation. 
Further, let $u_i$ given in \eqref{eq:HolonomCtrl} be the desired holonomic control direction for agent $i$.
The proposed control strategy is as follows. Each agent computes the control vector $u_i$ and projects it on its local heading and perpendicular heading directions. The projected vectors are then used as the linear and angular velocity commands. In the global coordinate frame, which is unknown to agent, this strategy can be described by
%
\begin{gather} \label{eq:kinemCtrl}
\begin{aligned}
v_i &:= h_i^\top \, u_i \\
\omega_i &:= h_i^{\perp \top} u_i ,
\end{aligned}
\end{gather}
%
as illustrated in Fig.~\ref{fig:control}.
\phantomsection \label{p:b9} Implementation of \eqref{eq:kinemCtrl} does not rely on a global coordinate system. 
This is because $h_i$ is a unit vector along the direction of vehicle, which is known by the agent locally, and $u_i$ is the single-integrator control given in the agent's local coordinate frame.

\begin{theorem} \label{thm:kinemP}
Let $A$ be a symmetric gain matrix designed for single-integrator agents. Under the control \eqref{eq:kinemCtrl}, unicycle agents globally converge to the desired formation.  
\end{theorem}

\begin{proof}
By replacing the control \eqref{eq:kinemCtrl} in  \eqref{eq:kinemUnicycle2}, the closed-loop dynamics can be expressed in the vector form by
%
\begin{gather} \label{eq:KinemDynam}
\begin{aligned}
\dot{q} &= H \, H^\top  A \, q \\
\dot{h} &= H^\perp  H^{\perp \top}   A \, q.
\end{aligned}
\end{gather}
%
Since $A$ is symmetric and negative semidefinite, we can consider 
%
\begin{gather} \label{eq:}
V := -\frac{1}{2} q^\top A \, q \geq 0
\end{gather}
%
as a Lyapunov function candidate. Time derivative of $V$ along the trajectory of \eqref{eq:KinemDynam} is
%
\begin{align} \label{eq:LyapAnalysis}	
\dot{V} &= - q^\top A \, \dot{q} \nonumber \\
&= - q^\top A \, H \, H^\top A \, q \nonumber  \\
&= - (H^\top A \, q)^\top (H^\top A \, q) = - \| H^\top A \, q \|^2 ~ \leq \, 0,
\end{align}
%
which implies that the system is stable. 
To show convergence to the desired formation we use the LaSalle's invariance principle and show that $q$ converges to the kernel of $A$. 
Since $\dot{V} = 0$ implies that $H^\top A\, q = 0$, by LaSalle's invariance principle $q$ converges to the largest invariant set in $\{q \in \br^{2n} \, |\, H^\top A\, q \equiv 0\}$. Thus, one of the following cases must hold:
\begin{enumerate}[label=(\roman*)]
	\item $A\, q \equiv 0$
	\item $A\, q \neq 0$, $H^\top A\, q \equiv 0$ 
\end{enumerate}	
Case (i) implies that the desired formation is achieved. In case (ii),  $H^\top A\, q \equiv 0$ implies that there exists constants $c_1,\, c_2,\, \dots,\, c_n \in \br$, with at least one $c_i \neq 0$, such that 
%
\begin{gather} \label{eq:Aq}	
A\, q = \begin{bmatrix}
c_1 h_1^\perp \\
c_2 h_2^\perp \\
\vdots \\
c_n h_n^\perp
\end{bmatrix} \neq 0.
\end{gather}
%
Since $H^\top A\, q \equiv 0$, from \eqref{eq:KinemDynam} we get $\dot{q} \equiv 0$. Thus, $q$ and $A\, q$ are constant, and from \eqref{eq:Aq} we conclude that $h_i^\perp$ (and thus $h_i$) is constant for all nonzero $c_i$. 
From the definition of $H^\perp$ in \eqref{eq:H}, one can see that $H^\perp$ has full column rank. Therefore, it does not have a right null vector, and from \eqref{eq:Aq} we have $H^{\perp \top} A \, q \neq0$. This shows $H^\perp H^{\perp \top} A \, q \neq 0$, and consequently from \eqref{eq:KinemDynam} we get $\dot{h} \neq 0$. This implies that the heading vectors are not fixed and rotating, which is a contradiction and shows that case (ii) cannot happen. 
\end{proof}

\begin{remark} 
From the closed-loop dynamics \eqref{eq:KinemDynam} one can see that when agents are at the desired formation, i.e., $A \, q = 0$, we have $\dot{h} = 0$ and hence the heading directions do not vary.  This implies that the controller drives agents to the desired formation, however their heading at the desired formation is not controlled and can take an arbitrary value. If desired, a supplementary control can be added to regulate heading angles after convergence.
\end{remark}

\begin{remark}
It is worth pointing out that the control \eqref{eq:kinemCtrl} can drive unicycle agents with a cart attached to the desired formation. In this case the position and orientation of the attached carts are not controlled.
The dynamics of a unicycle agent with cart attached is similar to the dynamics of a car, which is studied in the next section.
\end{remark}

\subsection{Robustness to Saturated Input}

In practice, the linear and angular velocities that an agent can execute are often limited to a certain range.  We show that under a saturated input, convergence of agents to the desired formation is not affected. 

\begin{theorem} \label{thm:UnicycleSaturated}
Consider the unicycle model \eqref{eq:kinemUnicycle2} and assume that $v_{\max},\, \omega_{\max} > 0$ are two real positive scalars. If $v_i$ and $\omega_i$ are saturated such that $|v_i| \leq v_{\max}$ , $|\omega_i| \leq \omega_{\max}$, then under the control \eqref{eq:kinemCtrl} the desired formation is achieved globally.  
\end{theorem}

\begin{proof}
To model the input saturation we can define the diagonal matrices $S,\, E  \in \br^{n \times n}$ with diagonal elements
%
\begin{gather} \label{eq:S}
(S)_{ii} = \begin{cases}
1 & \text{if} ~ |v_i| \leq v_{\max}  \\
\frac{v_{\max}}{|v_i|} & \text{if} ~ |v_i| > v_{\max}
\end{cases}
\end{gather}
%
and 
%
\begin{gather} \label{eq:E}
(E)_{ii} = \begin{cases}
1 & \text{if} ~ |\omega_i| \leq \omega_{\max}  \\
\frac{\omega_{\max}}{|\omega_i|} & \text{if} ~ |\omega_i| > \omega_{\max}.
\end{cases}
\end{gather}
%
Elements of $S,\, E$ can be considered as functions that saturate any large input to the maximum allowed values $v_{\max}, \, \omega_{max}$ (cf. Fig.~\ref{fig:Saturate} for saturated single-integrator control).
The closed-loop dynamics under the saturated input can be expressed in the vector form via
%
\begin{gather} \label{eq:UnicycleKinemSat}
\begin{aligned}
\dot{q} &= H \, S\, H^\top  A \, q \\
\dot{h} &= H^\perp E\, H^{\perp \top} \,  A \, q.
\end{aligned}
\end{gather}
%
System \eqref{eq:UnicycleKinemSat} should be understood as a family of switched dynamical systems, for which we choose $V := -\frac{1}{2} q^\top A \, q \geq 0$ as a common Lyapunov function candidate. Time derivative of $V$ along the trajectory of \eqref{eq:UnicycleKinemSat} is
%
\begin{align} \label{eq:VdotUnicycle}	
\dot{V} &= - q^\top A \, \dot{q} \nonumber \\
&= - q^\top A \, H \, S\,  H^\top A \, q \nonumber  \\
&= - (S^{\frac{1}{2}} \, H^\top A \, q)^\top (S^{\frac{1}{2}}\, H^\top A \, q) = - \| S^{\frac{1}{2}} \, H^\top A \, q \|^2 ~ \leq \, 0.
\end{align}
%
Thus, $V$ satisfies conditions of Lemma \ref{lem:SwitchedSys} and Corollary \ref{cor:SwitchedSysExtension}, and from LaSalle's invariance principle it follows that all trajectories of \eqref{eq:UnicycleKinemSat} converge to the zero set of $V$, which is the set of all desired formations. 
\end{proof}

\subsection{Robustness to Unmodeled Linear Actuator Dynamics} 

In practice, the linear and angular velocities of a vehicle cannot change instantaneously. The dynamic behavior of these velocities, which is not accounted for in the unicycle model \eqref{eq:kinemUnicycle}, can be modeled by
%
\begin{gather} \label{eq:DynamUnicycle}
\begin{aligned}
\dot{x}_i &= v_i \, \cos{(\theta_i)} \\
\dot{y}_i &= v_i \, \sin{(\theta_i)} \\
\dot{v}_i &= - a \,v_i + b \, s_i \\
\dot{\theta}_i &= \omega_i \\
\dot{\omega}_i &= - c \,\omega_i + d \, r_i 
\end{aligned}
\end{gather}
%
where $s_i,\, r_i \in \mathbb{R}$ are controls to adjust the linear and angular velocities, and $ a,b,c,d \in \mathbb{R}$ are strictly positive scalars, which depend on the vehicle's inertia, motor dynamics, friction, etc., and are in general unknown.   
We show that unmodeled velocity dynamics does not affect the convergence of the unicycle control strategy \eqref{eq:kinemCtrl}. That is, applying the control  
%
\begin{gather} \label{eq:DynamCtrl}
\begin{aligned}
s_i &:= h_i^\top \, u_i \\
r_i &:= h_i^{\perp \top} \, u_i
\end{aligned}
\end{gather}
%
in \eqref{eq:DynamUnicycle} results in the desired formation.

\begin{theorem} \label{thm:DynamUnicycle}
Let $A$ be a symmetric gain matrix designed for single-integrator agents.
Under the control \eqref{eq:DynamCtrl} agents with dynamics \eqref{eq:DynamUnicycle} globally converge to the desired formation.  
\end{theorem}

\begin{proof}
Substituting \eqref{eq:DynamCtrl} in \eqref{eq:DynamUnicycle} gives the closed-loop dynamics in the vector form as
%
\begin{gather} \label{eq:DynamDynam}
\begin{aligned}
\dot{q} &= H \, v \\
\dot{v} &= b\, H^\top A \, q - a \, v \\
\dot{h} &= H^\perp \, \omega\\
\dot{\omega} &= d \, H^{\perp \top} A \, q - c\, \omega
\end{aligned}
\end{gather}
%
where $v := [v_1,\, v_2,\, \dots v_n]^\top \in \mathbb{R}^n$ and $\omega := [\omega_1,\, \omega_2,\, \dots \omega_n]^\top \in \mathbb{R}^n$ are aggregate linear and angular velocity vectors, respectively. 
Consider the Lyapunov function candidate
%
\begin{gather} \label{eq:}
V := -\frac{b}{2} q^\top A \, q  + \frac{1}{2} v^\top v ~\geq~ 0.
\end{gather}
%
Time derivative of $V$ along the trajectory of \eqref{eq:DynamDynam} is
%
\begin{align} \label{eq:Vdot}	
\dot{V} &= - b\, q^\top A \, \dot{q}  +  \dot{v}^\top v \nonumber \\
&= - b\, q^\top A \, H \, v + b\, q^\top A \, H \, v - a \, v^\top v \nonumber  \\
&= - a \, v^\top v ~\leq ~ 0.
\end{align}
%
Similar to the proof of Theorem \ref{thm:kinemP}, we use  LaSalle's invariance principle and show that the largest invariant set consists of the desired formations. 
By setting $\dot{V} \equiv 0$ to find the invariant sets, from \eqref{eq:Vdot}	we get $v \equiv 0$, which implies that $\dot{v} \equiv 0$. Consequently, from \eqref{eq:DynamDynam} we should have that $b \, H^\top A \, q \equiv 0$, which implies one of the following two cases:
\begin{enumerate}[label=(\roman*)]
	\item $A\, q \equiv 0$
	\item $A\, q \neq 0$, $H^\top A\, q \equiv 0$.
\end{enumerate}	

Case (i) implies that the desired formation is achieved, where by replacing $A\, q \equiv 0$ in \eqref{eq:DynamDynam} the dynamics reduces to
%
\begin{gather} \label{eq:h}	
\begin{aligned}
\dot{h} &= H^\perp \omega \\
\dot{\omega} &= - c\, \omega.
\end{aligned}
\end{gather}
%
This shows $\omega,\, \dot{h} \rightarrow 0$, and therefore $\omega$ converges to zero and $h$ converges to a constant value. 
Thus, the set $\left\{ [q^\top,\, v^\top,\, g^\top,\, \omega^\top]^\top \in \mathbb{R}^{6n} \,:\, A\, q = 0,\, v = 0 \right\}$,
which consists of the desired formations, is an invariant set.

We now show that case (ii) cannot be an invariant set. Using a similar reasoning to the proof of Theorem \ref{thm:kinemP}, from $v \equiv 0$, $H^\top A\, q \equiv 0$, and dynamics \eqref{eq:DynamDynam} one can conclude that in this case $q$, $A\, q$, and $h$ are all constant and nonidentical to zero. Further,  $H^{\perp \top} A \, q \not\equiv 0$, 
which from \eqref{eq:DynamDynam} implies that $\dot{\omega} \not\equiv 0$ and hence $\omega \not\equiv 0$. This, together with $H^\perp$ having full column rank implies that $\dot{h} \not\equiv 0$, which is a contradiction to $h$ being constant. This shows that case (ii) is not an invariant set, which concludes the proof.
\end{proof}

\begin{remark} \label{rem:IdenticalParams}
In \eqref{eq:DynamUnicycle}, we assumed that $a,\, b,\, c,\, d$ have the same value for all agents. This assumption was made to simplify the notation and does not affect the generality of the results. One can assign a different value to these parameters for each agent and use the same analysis to prove the convergence.  
\end{remark}

\begin{remark} \label{rem:UnmodeledDynam}
In \eqref{eq:DynamUnicycle}, the assumption $a,\, c > 0$ implies that agents are zero-input stable, which often holds in practice. However, for $a,\, c < 0$ the control can be modified using the velocity feedback as
%
\begin{gather} \label{eq:DynamCtrl2}
\begin{aligned}
s_i &:= -k_s \, (v_i - h_i^\top \, u_i) \\
r_i &:= h_i^{\perp \top} \, u_i,
\end{aligned}
\end{gather}
%
where $k_s \in \mathbb{R}$ is a positive control gain. Using similar analysis to the proof of Theorem \ref{thm:DynamUnicycle}, one can show that if $k_s$ is chosen such that $a + k_s \, b > 0$, the agents converge to the desired formation. 
Lastly, with multiplying $s_i, \, r_i$ by the sign of $b,\, d$, respectively, the assumption $b,\, d > 0$ can be relaxed to only knowing the sign of these parameters. 
\end{remark}
\section{Formation Control for Agents with Car Dynamics} \label{sec:Car}

Cars are another common platform for which attaining a desired formation is often of interest (e.g., in intelligent transportation systems). In this section, we present a control strategy for agents with both front and rear-wheel drive car model. We then show that the convergence is not affected when the input is saturated, and the control is robust to unmodeled dynamics. 
Similar to previous section, henceforth we assume that a symmetric negative semi-definite control gain matrix $A$ is designed by solving the optimization problem \eqref{eq:OptimCVX}.

\begin{figure} 
\begin{center}	
	\includegraphics[trim = 90mm 70mm 85mm 60mm, clip, width=0.24\textwidth] {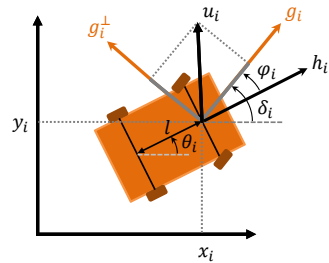}
	\caption{A car at position $(x_i,y_i)$ in the global coordinate frame. The agent's heading is denoted by $h_i$, and makes the angle $\theta_i$ with the global coordinate frame's $x$-axis. The front wheels' steering direction is along the vector $g_i$, which makes the angle $\delta_i$ with the $x$-axis.}
	\label{fig:CarModel}			
\end{center}
\end{figure}

\subsection{Control Strategy for Front-Wheel Drive Car}

Consider an agent with the front-wheel drive car model as illustrated in Fig.~\ref{fig:CarModel}. The motion of this agent can be described by the dynamics
%
\begin{gather} \label{eq:KinemCarFrontOriginal}
\begin{aligned}
\dot{x}_i &= v_i \, \cos(\theta_i + \varphi_i)\\
\dot{y}_i &= v_i \, \sin(\theta_i + \varphi_i)\\
\dot{\theta}_i &= \frac{v_i}{l} \, \sin(\varphi_i)  \\
\dot{\varphi}_i &= \omega_i
\end{aligned}
\end{gather}
%
where $x_i,\, y_i \in \br^2$ are the coordinates of the front axle's center, $v_i \in \br$ is the driving velocity, $\theta_i \in [0,\, 2\pi)$ is the heading angle, $\varphi_i \in [0,\, 2\pi)$ is the steering angle, $\omega_i$ is the steering velocity, and $l \in \br$ is the wheelbase. In this kinematic model, it is assumed that $v_i$ and $\omega_i$ are inputs and can be controlled directly. By defining 
%
\begin{gather} \label{eq:mu}
\delta_i := \theta_i + \varphi_i,
\end{gather}
%
one can alternatively write \eqref{eq:KinemCarFrontOriginal} as \cite{DeLuca1998}
%
\begin{gather} \label{eq:KinemCarFront}
\begin{aligned}
\dot{x}_i &= v_i \, \cos(\delta_i) \\
\dot{y}_i &= v_i \, \sin(\delta_i) \\
\dot{\theta}_i &= \frac{v_i}{l} \, \sin(\delta_i- \theta_i) \\
\dot{\delta}_i &= \frac{v_i}{l} \, \sin(\delta_i- \theta_i) + \omega_i.
\end{aligned}
\end{gather}
%
Note that to simplify the notation, we have assumed that $l$ is identical for all agents. This does not affect the generality of the following results, and one can carry the following analysis with a different $l$ for each agent.

To derive an alternative formulation for \eqref{eq:KinemCarFront} that is more suitable for the control design, we define the steering vector $g_i \in \br^2$ and its perpendicular $g_i^\perp \in \br^2$, and heading vector $h_i \in \br^2$ and its perpendicular $h_i^\perp \in \br^2$ as
%
\begin{gather} \label{eq:hdef2}
\begin{aligned}
g_i &:= \begin{bmatrix}
\cos(\delta_i)\\
\sin(\delta_i)
\end{bmatrix}, \qquad
g_i^\perp := \begin{bmatrix}
-\sin(\delta_i) \\
\cos(\delta_i)
\end{bmatrix}, \\
h_i &:= \begin{bmatrix}
\cos(\theta_i)\\
\sin(\theta_i)
\end{bmatrix}, \qquad
h_i^\perp := \begin{bmatrix}
-\sin(\theta_i) \\
\cos(\theta_i)
\end{bmatrix}.
\end{aligned}
\end{gather}
%
Seeing that $\dot{g}_i = g_i^\perp \dot{\delta}_i$, $\dot{h}_i = h_i^\perp \dot{\theta}_i$, and $\sin(\delta_i - \theta_i) = \sin(\delta_i) \cos(\theta_i) - \cos(\delta_i) \sin(\theta_i) = h_i^{\perp \top} g_i$, we can describe \eqref{eq:KinemCarFront} equivalently by
%
\begin{gather} \label{eq:KinemCarFrontVector}
\begin{aligned}
\dot{q}_i &= g_i \, v_i  \\
\dot{g}_i &= g_i^\perp  \, \left(\frac{1}{l}\, h_i^{\perp \top} g_i \, v_i + \omega_i \right)\\
\dot{h}_i &=  h_i^\perp \left( \frac{1}{l}\, h_i^{\perp \top} g_i \, v_i \right).
\end{aligned}
\end{gather}
%
From \eqref{eq:KinemCarFrontVector}, the dynamics of all agents can be collectively expressed in the vector form
%
\begin{gather} \label{eq:KinemCarFrontMatrix}
\begin{aligned}
\dot{q} &= G\, v \\
\dot{g} &= \frac{1}{l} G^\perp H^{\perp \top} G  \, v + G^\perp \, \omega  \\
\dot{h} &= \frac{1}{l} H^\perp H^{\perp \top} G  \, v.
\end{aligned}
\end{gather}
%
where $q,\, g,\, h \in \br^{2n}$ are aggregate state, steering, and heading vectors, and $v,\, \omega \in \br^n$ are aggregate control vectors. Further, $H,\, H^\perp$ are defined according to \eqref{eq:H}, and $G, \, G^\perp$ are defined by replacing $h_i$'s by $g_i$'s in \eqref{eq:H}.

Using a similar strategy to the unicycle agents in Section \ref{sec:Unicycle}, we define the driving and steering velocity controls as the projections of the holonomic control vector along the steering direction and its perpendicular by
%
\begin{gather} \label{eq:CarKinemControl}
\begin{aligned}
v_i &:= g_i^\top \, u_i  \\
\omega_i &:= g_i^{\perp \top}  \, u_i,
\end{aligned}
\end{gather}
%
where $u_i$ is given in \eqref{eq:HolonomCtrl}. We emphasize that \eqref{eq:CarKinemControl} can be implemented using only the local relative position measurements.

\begin{theorem} \label{thm:CarFrontP}
Let $A$ be a symmetric gain matrix designed for single-integrator agents. Under the control \eqref{eq:CarKinemControl}, agents with front-wheel drive car dynamics globally converges to the desired formation.  
\end{theorem}

\begin{proof}
The proof follows from similar analysis to the proof of Theorem \ref{thm:kinemP}. By substituting \eqref{eq:CarKinemControl} in \eqref{eq:KinemCarFrontMatrix}, the closed-loop dynamics is given in the vector form as
%
\begin{gather} \label{eq:KinemCarFrontClosedLoop}
\begin{aligned}
\dot{q} &= G\, G^\top  A \, q \\
\dot{g} &=  \frac{1}{l} G^\perp H^{\perp \top} G \, G^\top  A \, q + G^\perp \, G^{\perp \top}  A \, q  \\
\dot{h} &= \frac{1}{l} H^\perp H^{\perp \top} G \, G^\top  A \, q.
\end{aligned}
\end{gather}
%
Using $V := -\frac{1}{2} q^\top A \, q \geq 0$ as a Lyapunov function candidate, one can show that the time derivative of $V$ along the trajectory of \eqref{eq:KinemCarFrontClosedLoop} is $\dot{V} = - \| G^\top A \, q \|^2 ~ \leq \, 0$,
which implies the stability of system. Convergence to the desired formation follows from the LaSalle's invariance principle. In particular, in the case that $A \,q \neq 0$ but $G^\top A \, q \equiv 0$, dynamics of $g$ in \eqref{eq:KinemCarFrontClosedLoop} reduces to $\dot{g} =  G^\perp \, G^{\perp \top}  A \, q $, which is the same as dynamics for $h$ in the unicycle model \eqref{eq:KinemDynam}. By the same token, this case cannot be a invariant set, and the only possibility is $A\, q \equiv 0$, which indicates that the desired formation is achieved. 
\end{proof}

\begin{remark} 
Similar to the unicycle agents, the final heading and steering angles of agents with car dynamics at the the desired formation are not controlled and can take arbitrary values.
\end{remark}

\subsection{Control Strategy for Rear-Wheel Drive Car}

The dynamics of a rear-wheel drive car is identical to the front-wheel drive car except that the front wheels' driving velocity $v_i$ is indirectly controlled via the rear wheels' driving velocity $v_i^r$. The relation between the front and rear wheels' driving velocities is given by
%
\begin{gather} \label{eq:RearWheelSpeed}	
v_i = \frac{1}{\cos(\varphi_i)} \, v_i^r.
\end{gather}
%
To set $v_i$ to the desired value defined in \eqref{eq:CarKinemControl}, from \eqref{eq:RearWheelSpeed} we have that the rear wheels' driving velocity should be  
%
\begin{gather} \label{eq:CarRearSpeedControl}
v_i^r := \cos(\varphi_i) \,  g_i^\top u_i.
\end{gather}
%
The main difference between the rear and front-wheel drive car is that when $\varphi_i = \pm \frac{\pi}{2}$, from \eqref{eq:CarRearSpeedControl}  $v_i^r$, and hence $v_i$, become zero. On the contrary, $v_i$ in a front-wheel drive car can take any desired value in this case (one can interpret this as the car pivoting about its rear wheels).

\begin{theorem} \label{thm:CarRearP}
Under the conditions of Theorem \ref{thm:CarFrontP} with driving velocity control \eqref{eq:CarRearSpeedControl}, agents with rear-wheel drive car dynamics almost globally converge to the desired formation.
\end{theorem}

\begin{proof}
\phantomsection \label{p:b10} Under the control \eqref{eq:CarRearSpeedControl}, the closed-loop dynamics is similar to \eqref{eq:KinemCarFrontClosedLoop}, except when the steering angles are $\pm \frac{\pi}{2}$, in which case the driving velocity is zero. By defining the diagonal matrix $\Gamma \in \br^{n\times n}$ with diagonal entries 
%
\begin{gather} \label{eq:}
(\Gamma)_{ii} = \begin{cases}
1 & \text{if} ~ \varphi_i \neq \pm \frac{\pi}{2} \\
0 & \text{if} ~ \varphi_i = \pm \frac{\pi}{2} 
\end{cases}
\end{gather}
%
driving velocity can be expressed as ${v = \Gamma\, G^\top A \, q}$, and from \eqref{eq:KinemCarFrontMatrix} the closed-loop dynamics of a rear-wheel drive car is given by
%
\begin{gather} \label{eq:KinemCarRearClosedLoop}
\begin{aligned}
\dot{q} &= G\, \Gamma\, G^\top  A \, q \\
\dot{g} &= \frac{1}{l}  G^\perp  H^{\perp \top} \, G \, \Gamma\, G^\top  A \, q + G^\perp \, G^{\perp \top}  A \, q  \\
\dot{h} &= \frac{1}{l} H^\perp H^{\perp \top} \, G \, \Gamma\, G^\top  A \, q.
\end{aligned}
\end{gather}
%
We use $V := -\frac{1}{2} q^\top A \, q \geq 0$ as a common Lyapunov function candidate for the switched system \eqref{eq:KinemCarRearClosedLoop} to prove stability and convergence in a manner similar to Theorem \ref{thm:CarFrontP}. By direct calculation, derivative of $V$ along the trajectory of \eqref{eq:KinemCarRearClosedLoop} is $ \dot{V} = \| \Gamma\, G^\top A \, q \|^2 ~ \leq \, 0$. 
When the diagonal elements of $\Gamma$ are all ones, i.e., no heading angle is equal to $\pm \frac{\pi}{2}$, the dynamics \eqref{eq:KinemCarRearClosedLoop} is identical to \eqref{eq:KinemCarFrontClosedLoop} and convergence follows from the proof of Theorem \ref{thm:CarFrontP}. Thus, we only need to analyze instances where $\varphi_i = \pm \frac{\pi}{2}$. At such instances, one of the following cases hold
\begin{enumerate}[label=(\roman*)]
	\item $\exists i,  \varphi_i \neq \pm \frac{\pi}{2}$ or $G^{\perp \top} A\, q \neq 0$
	\item $\forall i, \varphi_i = \pm \frac{\pi}{2}$ and $G^{\perp \top} A\, q = 0$.
\end{enumerate}	
If agents are not at the desired formation, i.e., $A\, q \neq 0$, case (i) cannot be an invariant set. This is because $\dot{V} \equiv 0$ implies $\Gamma\, G^\top A \, q \equiv 0$, and hence from \eqref{eq:KinemCarRearClosedLoop} we get $\dot{g} = G^\perp \, G^{\perp \top}  A \, q $, which shows $g$ is varying and the heading angles cannot remain at $\pm \frac{\pi}{2}$. 
On the other hand, case (ii) is an invariant set at which the agents stop moving without reaching the desired formation. From the Picard-Lindelof theorem on the existence and uniqueness of solutions, only one trajectory of system \eqref{eq:KinemCarFrontClosedLoop} passes through the point where all $\varphi_i$'s are $\frac{\pi}{2}$. Thus, the number of trajectories at which  all heading angles are either $\frac{\pi}{2}$ or $-\frac{\pi}{2}$ is $2^n$. In the space of all trajectories, these trajectories are a measure zero set (i.e., they have zero volume). 
This shows almost global convergence of system \eqref{eq:KinemCarRearClosedLoop} to the desired formation. 
\end{proof}

\begin{remark} 
Due to noise, unmodeled dynamics, disturbances, etc., in practice agents cannot stay on a measure zero set of trajectories. Furthermore, as we will show subsequently, the steering angle of a car can be bounded to remain less than $\pm \frac{\pi}{2}$ and avoid the undesired case (ii) in the proof.
Consequently, the ``almost" global convergence of rear-wheel drive car in Theorem \ref{thm:CarRearP} does not affect the applicability of the control strategy. 
\end{remark}

Due to the similarity of the dynamics and analysis for the front and rear-wheel drive car models, we only consider the front-wheel drive car model throughout the rest of this section.

\subsection{Robustness to Saturated Input and Bounded Steering Angle}

The steering angle and the driving and steering velocities of a car often cannot take arbitrary values and must be bounded by practical limits. This, however, does not affect the convergence of the agents to the desired formation.

\begin{theorem} \label{thm:CarPSaturated}
Consider car dynamics \eqref{eq:KinemCarFrontOriginal}, and assume that $v_{\max}, \, \omega_{\max}, \, \varphi_{\max} > 0$ are real positive scalars. If $v_i$ and $\omega_i$ are saturated such that $|v_i| \leq v_{\max}$ , $|\omega_i| \leq \omega_{\max}$, and the steering angle is bounded by $|\varphi_i| \leq \varphi_{\max}$, then under the control \eqref{eq:CarKinemControl} the desired formation is achieved globally. 
\end{theorem}

\begin{proof}
To model the input saturation, we consider the diagonal matrices $S,\, E \in \br^{n\times n}$ defined in \eqref{eq:S}, \eqref{eq:E}. Further, to model the bounded steering angel we define the diagonal matrix $\Gamma \in \br^{n\times n}$ via
%
\begin{gather} \label{eq:}
(\Gamma)_{ii} = \begin{cases}
1 & \text{if} ~ |\varphi_i| \leq \varphi_{\max}  \\
0 & \text{if} ~ |\varphi_i| > \varphi_{\max}
\end{cases}
\end{gather}
%
The closed-loop dynamics under the saturated input and bounded steering angle can be expressed in  vector form as
%
\begin{gather} \label{eq:CarSaturated}
\begin{aligned}
\dot{q} &= G\, S\, G^\top  A \, q \\
\dot{g} &= \frac{1}{l} G^\perp H^{\perp \top} G \, S\, G^\top  A \, q + G^\perp \, \Gamma \, E\, G^{\perp \top} A \, q  \\
\dot{h} &= \frac{1}{l} H^\perp H^{\perp \top} G \, S\, G^\top  A \, q.
\end{aligned}
\end{gather}
%
The solutions of switched system \eqref{eq:CarSaturated} are well-defined in the Filippov sense. Similar to the proof of Theorem \ref{thm:UnicycleSaturated}, by considering $V := -\frac{1}{2} q^\top A \, q \geq 0$ as a common Lyapunov function candidate, the time derivative of $V$ along the trajectory of \eqref{eq:KinemCarRearClosedLoop} is  $\dot{V} = - \| S^{\frac{1}{2}} \, G^\top A \, q \|^2 ~ \leq \, 0$.
The Lyapunov function $V$ satisfies the conditions of Lemma~\ref{lem:SwitchedSys}, and from Corollary \ref{cor:SwitchedSysExtension} and LaSalle's invariance principle, it follows that the desired formation is achieved.
\end{proof}

\subsection{Robustness to Unmodeled Linear Actuator Dynamics}

Since in practice the driving and steering velocities of a car
cannot change instantaneously, the car dynamics \eqref{eq:KinemCarFrontOriginal} can be modified as
%
\begin{gather} \label{eq:DynamCarFrontOriginal}
\begin{aligned}
\dot{x}_i &= v_i \, \cos(\theta_i + \varphi_i)\\
\dot{y}_i &= v_i \, \sin(\theta_i + \varphi_i)\\
\dot{v}_i &= - a \,v_i + b \, s_i \\
\dot{\theta}_i &= \frac{v_i}{l} \, \sin(\varphi_i)  \\
\dot{\varphi}_i &= \omega_i \\
\dot{\omega}_i &= - c \,\omega_i + d \, r_i
\end{aligned}
\end{gather}
%
to incorporate the dynamics of these velocities.
In \eqref{eq:DynamCarFrontOriginal},  $s_i, \, r_i \in \mathbb{R}$ are control inputs to adjust the driving and steering velocities. Further, we assume that $ a,b,c,d \in \mathbb{R}$ are strictly positive, but unknown. 
The following theorem shows that the unmodeled velocity dynamics does not affect the convergence of the control strategy \eqref{eq:CarKinemControl}.
That is, applying the control
%
\begin{gather} \label{eq:DynamCarCtrl}
\begin{aligned}
s_i &:= g_i^\top \, u_i \\
r_i &:= g_i^{\perp  \top} \, u_i
\end{aligned}
\end{gather}
%
in \eqref{eq:DynamCarFrontOriginal} results in the desired formation.

\begin{theorem} \label{thm:CarFrontDynam} 
Let $A$ be a symmetric gain matrix designed for single-integrator agents.
Under the control \eqref{eq:DynamCarCtrl}, agents with dynamics \eqref{eq:DynamCarFrontOriginal} globally converge to the desired formation.  
\end{theorem}

\begin{proof}     
By substituting \eqref{eq:DynamCarCtrl} in \eqref{eq:DynamCarFrontOriginal}, the closed-loop dynamics is given in the vector from by 
%
\begin{gather} \label{eq:DynamCarFront}
\begin{aligned}
\dot{q} &= G\, v \\
\dot{v} &=  b\,  G^\top A\, q - a \, v   \\
\dot{g} &= \frac{1}{l} G^\perp H^{\perp \top} G  \, v + G^\perp \, \omega  \\
\dot{h} &= \frac{1}{l} H^\perp H^{\perp \top} G  \, v \\
\dot{\omega} &= d\, G^{\perp \top} A\, q -c \, \omega \\
\end{aligned}
\end{gather}
%
where $v, \, \omega$ are the aggregate driving and steering velocity vectors.
Similar to the proof of Theorem \ref{thm:DynamUnicycle}, we consider $V := -\frac{b}{2} q^\top A \, q  + \frac{1}{2} v^\top v ~ \geq 0$ as a Lyapunov function candidate. After simplifications, the time derivative of $V$ along the trajectory of \eqref{eq:DynamCarFront} is given by $- a \, v^\top v \leq 0$. 
To show convergence using LaSalle's invariance principle, we set $\dot{V} \equiv 0$ to find the largest invariant set. This implies that $v \equiv 0$, and therefore $\dot{v} \equiv 0$. 
Consequently, from \eqref{eq:DynamCarFront} we should have
that $b \, G^\top A \, q \equiv 0$, which implies one of the following two cases:
\begin{enumerate}[label=(\roman*)]
	\item $A\, q \equiv 0$
	\item $A\, q \neq 0$, $G^\top A\, q \equiv 0$.
\end{enumerate}	

Case (i) implies that the desired formation is achieved, where by replacing $v \equiv 0$ and $A\, q \equiv 0$ in \eqref{eq:DynamCarFront} the dynamics reduces to
%
\begin{gather} \label{eq:h}	
\begin{aligned}
\dot{g} &= G^\perp \omega \\
\dot{\omega} &= - c\, \omega.
\end{aligned}
\end{gather}
%
This shows $\omega,\, \dot{g} \rightarrow 0$, and therefore $\omega$ converges to zero and $g$ converges to a constant value. 
Thus, the set $\left\{ [q^\top,\, v^\top,\, g^\top,\, \omega^\top]^\top \in \mathbb{R}^{6n} \,:\, A\, q = 0,\, v = 0 \right\}$,
which consists of the desired formations, is an invariant set.

We now show that case (ii) cannot be an invariant set. Using a similar reasoning to the proof of Theorem \ref{thm:DynamUnicycle}, from $v \equiv 0$, $G^\top A\, q \equiv 0$, and dynamics \eqref{eq:DynamCarFront} one can conclude that in this case $q$, $A\, q$, and $g$ are all constant and nonidentical to zero. Further,  $G^{\perp \top} A \, q \not\equiv 0$, 
which from \eqref{eq:DynamCarFront} implies that $\dot{\omega} \not\equiv 0$ and hence $\omega \not\equiv 0$. This, together with $G^\perp$ having full column rank implies that $\dot{g} \not\equiv 0$, which is a contradiction to $g$ being constant. This shows that case (ii) is not an invariant set, which concludes the proof.
\end{proof}

\begin{remark} 
On a similar note to Remarks \ref{rem:IdenticalParams} and \ref{rem:UnmodeledDynam}, in \eqref{eq:DynamCarFrontOriginal} parameters $a,\, b,\, c,\, d$ can take different values for each agent.  Further, if $k_s \in \br$ is chosen such that $a + k_s \, b >  0$, the modified control
%
\begin{gather} \label{eq:DynamCarCtrl2}
\begin{aligned}
s_i &:= -k_s (v_i - g_i^\top \, u_i) \\
r_i &:= g_i^{\perp  \top} \, u_i
\end{aligned}
\end{gather}
%
can bring agents with $a, \, c <0$ to the desired formation.
\end{remark}

\section{Extensions and Variations} \label{sec:Extensions}

In this section, we briefly address additional topics such as collision avoidance, stability under a time-varying sensing topology, and formation scale adjustment, which are important in a practical implementation.

\subsection{Collision Avoidance} \label{sec:ColAvoid}

As discussed in Section \ref{sec:Robustness}, by using the control gains computed from Algorithm~\ref{alg:GainDesign}, positive scaling and rotation of the the control vectors up to $\pm 90^\circ$ does not affect convergence of agents to the desired formation. 
This observation can be used to implement a distributed collision avoidance strategy. 
Fig.~\ref{fig:ColAvoidance} illustrates a scenario where the desired control direction of agent $i$ can potentially cause collision with an adjacent agent. 
Tangent lines from agent $i$ to circles of radius $r \in \br$ centered at the adjacent agents define collision avoidance cones, where by rotating the control vector to a direction outside of these cones the collision is prevented. 
To preserve the stability properties, the rotation is limited to $\pm 90^\circ$ of the original control direction, e.g., in Fig.~\ref{fig:ColAvoidance} the control $u_i$ cannot be rotated below the solid black line. In a case where there is no possible direction of motion within the allowed rotation range the control is set to zero, and the agent stops until a feasible control direction is available. 
To alter the control direction as little as possible, one can define a distance threshold $d_c \in \br$ such that the collision avoidance strategy is triggered only when the distance to an adjacent agent is less than this threshold. 
The above collision avoidance strategy is outlined in Algorithm~\ref{alg:ColAvoidance}. Note that this strategy is distributed and does not required inter-agent communication.

If the initial inter-agent distances are greater than $r$, it is straightforward to show that collision avoidance is guaranteed for \textit{single-integrator} agents under Algorithm~\ref{alg:ColAvoidance} 
(this follows by showing that inter-agent distances cannot become less than $r$). 
\phantomsection \label{p:b4} For agents with \textit{higher-order} dynamics, $r$ should be chosen large enough to accommodate for the maximum braking distance.
We should point out that convergence to the desired formation under the proposed strategy is heuristic and not always guaranteed. In particular, one can construct counter examples where agents are caught in a gridlock due to unavailability of feasible control direction.
However, in our simulation and experimental studies we observed that if agents are initially spaced far apart, they can resolve gridlocks and converge to the desired formation. 
We are not aware of any existing collision avoidance strategy that is distributed, does not require communication, and can guarantee convergence of agents to the desired formation. Commonly used distributed strategies such as safety barrier functions \cite{Wang2017} and traffic circles \cite{Jin2015} have similar gridlock situations.
We point out that in scenarios where inter-agent communication is possible, distributed task assignment techniques can be leveraged to resolve gridlocks (see our recent work \cite{Lusk2020}).

\phantomsection \label{p:b3} 
On the other hand, under the proposed strategy, stability of the overall system is guaranteed by Theorem \ref{thm:Robustness}. 
This distinguishes the proposed strategy from an ad hoc augmentation of the control to avoid collision, e.g., via  potential functions. Such augmentations may lead to undesired behavior or instability of the overall system. For example, they may cause the robots to drift along a direction or circle in a limit cycle indefinitely. Such behaviors are not present in the proposed approach, and \textit{if} the robots do not go to a \textit{gridlock}, convergence to the desired shape is guaranteed.

\begin{figure} 
	\begin{center}	
		\includegraphics[trim = 84mm 70mm 95mm 60mm, clip, width=0.22\textwidth] {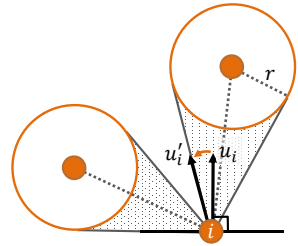}
		\caption{Control vector of agent $i$ rotated outside of the collision cones.}	
		\label{fig:ColAvoidance}			
	\end{center}
\end{figure}

\begin{algorithm}[t] 
%
\DontPrintSemicolon
\SetKwData{Left}{left}\SetKwData{This}{this}\SetKwData{Up}{up}
\SetKwFunction{Union}{Union}\SetKwFunction{FindCompress}{FindCompress}
\SetKwInOut{Input}{input}\SetKwInOut{Output}{output}

\SetKwInput{StepA}{step 1}
\SetKwInput{StepB}{step 2}
\SetKwInput{StepC}{step 3}
\SetKwInput{StepD}{step 4}
\SetKwInput{Notation}{notation}

\caption{Distributed collision avoidance.}

\Input{Desired control direction $u_i \in \br^2$\\
Collision circle radius $r  \in \br$ \\
Activation threshold $d_c  \in \br$}
\Output{Modified control direction $u'_i  \in \br^2$}

\BlankLine

\StepA{Construct collision cones with circles of radius $r$ centered at agents closer than $d_c$.} 
\StepB{Find rotation $R(\theta) \in \mathrm{SO}(2)$ with minimum $|\theta|$ such that $R \, u_i$ is outside of collision cones.}
\StepC{If step~2 is infeasible or $|\theta| \geq 90^\circ$ set $u'_i = 0$, otherwise set $u'_i = R(\theta) \, u_i$.}
%
\label{alg:ColAvoidance}
\end{algorithm}

\subsection{Time-Varying Sensing Topology}

In a time-varying or switching sensing topology, the agents can lose or acquire sensing capability of other agents in the group. For example, if a vision sensor is used to provide position measurements, sensing capability is lost when a neighbor agent is obstructed by another agent and acquired when an agent moves in the line of sight.
The following theorem shows that by using the proposed gain design approach a switching sensing topology does not affect the convergence of single-integrator agents to the desired formation.

\begin{theorem} \label{thm:Switching} 
Let $\mathcal{G} := \{\mathcal{G}^1, \mathcal{G}^2, \dots, \mathcal{G}^m \}$ denote a finite set of undirected and universally rigid sensing topologies, with associated gain matrices $A^1,\,  A^2,\, \dots,\, A^m \in \br^{2n \times 2n}$ computed from Algorithm~\ref{alg:GainDesign}. 
If single-integrator agents use the associated gains for each topology,
under control \eqref{eq:HolonomCtrl} the agents globally converge to the desired formation while the sensing graph can switch in $\mathcal{G}$ arbitrarily. 
\end{theorem}

\begin{proof}     
The closed-loop dynamics under the proposed control strategy is given by $\dot{q} = A^i \, q$, where $i \in {1,\, 2,\, \dots, m}$ denote the index of the sensing topology. By considering  $V := q^\top \, q \, \geq 0$ as a common Lyapunov function candidate for the this family of switched systems, we derive $\dot{V} = q^\top \, A^i \, q$. Since $A^i$ is negative semidefinite by design for every $i$, it follows that $\dot{V} \leq 0$. Hence, from Lemma~\ref{lem:SwitchedSys} and Corollary~\ref{cor:SwitchedSysExtension} we have that the desired formation is achieved under an arbitrary switching among topologies. 
\end{proof}

Theorem~\ref{thm:Switching} ensures convergence under an arbitrary switching of sensing topologies provided that stabilizing gain matrices are computed for each topology.  
To ensure that the formation control strategy is applicable in a switching scenario \textit{without}  inter-agent communication, additional constraints can be enforced to obtain gains that jointly stabilize all sensing topologies. 
To elaborate this point, consider the example of four sensing topologies illustrated in Fig.~\ref{fig:Graphs}. In topologies numbered as (1) and (3), agent 1 has the same set of neighbors, namely agents 2 and 4. Since agent 1 is not aware of the overall sensing topology, from its point of view topologies (1) and (3) are indistinguishable. Consequently, the control gains for agent 1 in matrices $A^1$ and $A^3$ should be identical to ensure they jointly stabilize both topologies. 

To find gain matrices that jointly stabilize switching sensing topologies, the optimization problem \eqref{eq:OptimCVX} can be modified as follows. 
We define the block diagonal matrix $\Lambda \in \br^{2nm \times 2nm}$ as ${\Lambda := \mathrm{diag}(\bar{A}^1,\, \bar{A}^2,\, \dots, \, \bar{A}^m)}$, where $\bar{A}^k$ is defined according to \eqref{eq:Lbar}. The gain matrices that jointly stabilize the topologies are found by solving
%
\begin{gather} \label{eq:OptimCVXJoint}
\begin{array}{lll}
\underset{a^k_{ij},\, b^k_{ij}}{\mathrm{maximize}} & \lambda_{1} (-\Lambda)  & \\
\text{subject to} &  A^k \, N = 0 & \forall_{k\in \bn_m} \\
& \mathrm{trace}(\Lambda) = \text{constant} & \\
& \mathcal{A}(\Lambda) = 0 &
\end{array}
\end{gather}
%
Here, $a^k_{ij},\, b^k_{ij}$ are entries of $A^k$, the first constraint ensures that $N$ is the kernel of all gain matrices, and the second constraint ensures that the problem is bounded.  The expression $\mathcal{A}(\Lambda) = 0$ encapsulates the constraints that enforce the block diagonal structure of $\Lambda$ and ensure agents with identical set of neighbors in two (or more) topologies have the same set of gains.

In a manner similar to problem~\eqref{eq:OptimCVX}, the objective of \eqref{eq:OptimCVXJoint} aims to minimize the largest eigenvalue of all gain matrices (note that eigenvalues of a block diagonal matrix consist of the eigenvalues of each diagonal block). 
While universal rigidity of the sensing graph is necessary and sufficient to ensure Algorithm~\ref{alg:GainDesign} results in a stabilizing gain matrix,  
to ensure a group of gain matrices are jointly stabilizing additional sensing is often required. A sufficient conditions under which joint stabilizability is guaranteed is provided in our prior work \cite[see Thm. 4]{Fathian2017}, which depends on the number and topology of the sensing graphs.

While the focus of this work is formation control without inter-agent communication, we point out that in scenarios where communication is possible, other techniques can be leveraged to handle switching topologies. 
For example, the gains can be found online for each topology via solving \eqref{eq:OptimCVX} using a distributed ADMM techniques \cite{Boyd2011}, which requires inter-agent communication to converge. 
Ultimately, the appropriate method for handling a switching scenario depends on the hardware and communication constraints.

\phantomsection \label{p:b1} 
Lastly, we emphasize that the result of Theorem~\ref{thm:Switching} are based on the single-integrator dynamics. 
Due to the convergence properties of control for unicycle and car dynamics in a fixed topology, under suitable assumptions that switching is slow enough (i.e., large dwell time), convergence of unicycles and cars to the desired formation in the switching case can be expected.
Deriving a lower bound for the dwell time will be a topic of future work.

\subsection{Scale Adjustment} \label{sec:Scale}

To fix the scale of the final formation, control law  \eqref{eq:HolonomCtrl} can be augmented by a bounded smooth map $f: \mathbb{R} \rightarrow \mathbb{R}$ as
%
\begin{equation} \label{eq:HolonomCtrlAugment}
u_i = \sum_{j \in \mathcal{N}_i} A_{ij} \, (q_j-q_i)  + f(d_{ij} - d_{ij}^*)\, (q_j-q_i),
\end{equation}
%
where $d_{ij} := \|q_j - q_i\|$ denote the distance between agent $i$ and $j$,  $d_{ij}^* \in \mathbb{R}$ is its desired value, and  $f$ is chosen such that  $x \, f(x) > 0$ for ${x \neq 0}$, and $f(0) = 0$. Possible choices for $f$ are $f: x \mapsto \frac{1}{k}\arctan(x)$ or $f: x \mapsto \frac{1}{k}\tanh(x)$, where $k > 0$ is an arbitrary constant. The role of $f$ in \eqref{eq:HolonomCtrlAugment} is to pull agents toward their neighbors when the distance between them is larger than the desired value, and vice versa. 
For agents with single-integrator dynamics, we have shown that agents almost \textit{globally} converge to the desired formation \cite{Fathian2019}.
The study of global asymptotic stability for agents with higher order dynamics is a topic of future research.

\subsection{3D Formations} \label{sec:3dform}

\phantomsection \label{p:b5} 
The proposed control approach, together with the convergence and robustness properties, can be extended to 3D formations. This extension has been done in our recent work \cite{Fathian2019a}, where experimental validations on a fleet of Crazyflie quadrotors are performed to demonstrate the strategy.

\section{Simulations} \label{sec:Simulations}

To validate the proposed approach, we present several simulations for planar formation of quadrotors, unicycles, and cars. 
Links to simulation code and videos are provided in the Supplementary Material section.

\begin{figure}
	\begin{center}
		\includegraphics[trim =25mm 75mm 23mm 87mm, clip, width=0.45\textwidth]{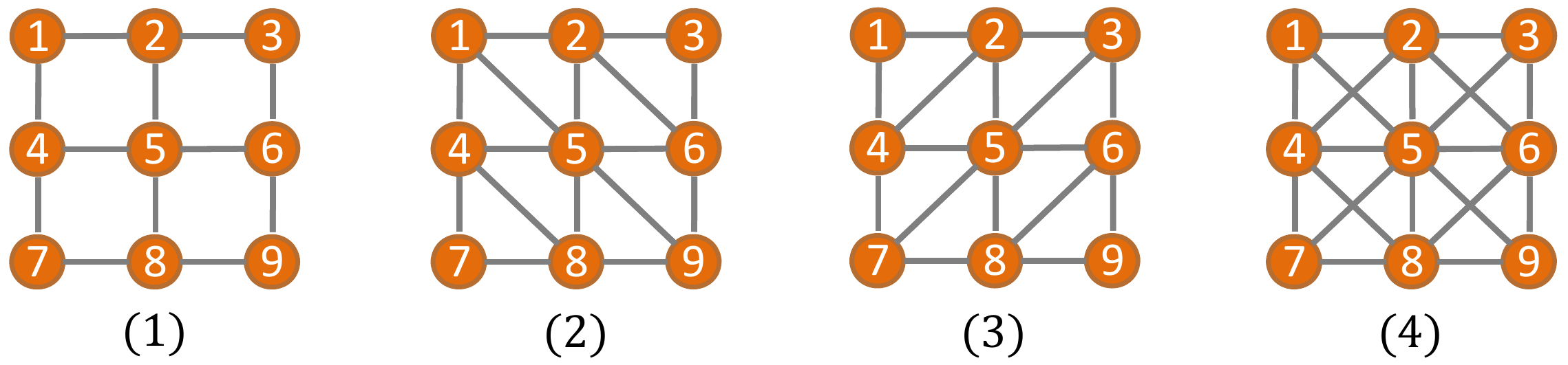}	
		\includegraphics[trim =5mm 0mm 5mm 2mm, clip, width=0.42\textwidth]{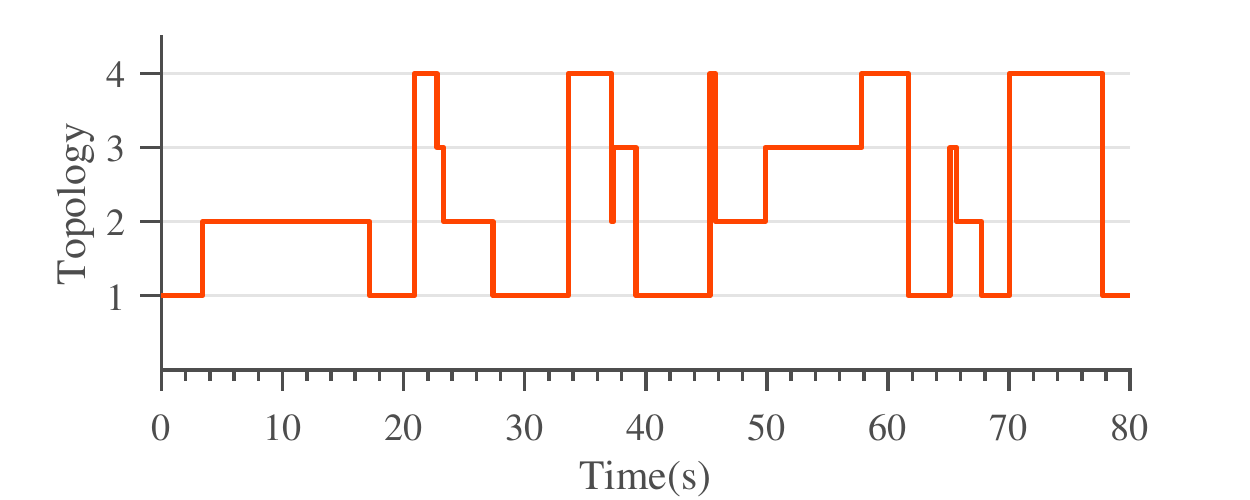}
		\caption{Top: Four sensing topologies. Bottom: Switching among the topologies vs. time.}
		\label{fig:Graphs}
	\end{center}
\end{figure}

\begin{figure*}[]
	\centering
	\begin{subfigure}[b]{.2\textwidth}
		\includegraphics[trim = 1mm 1mm 1mm 1mm, clip, width=0.97\textwidth] {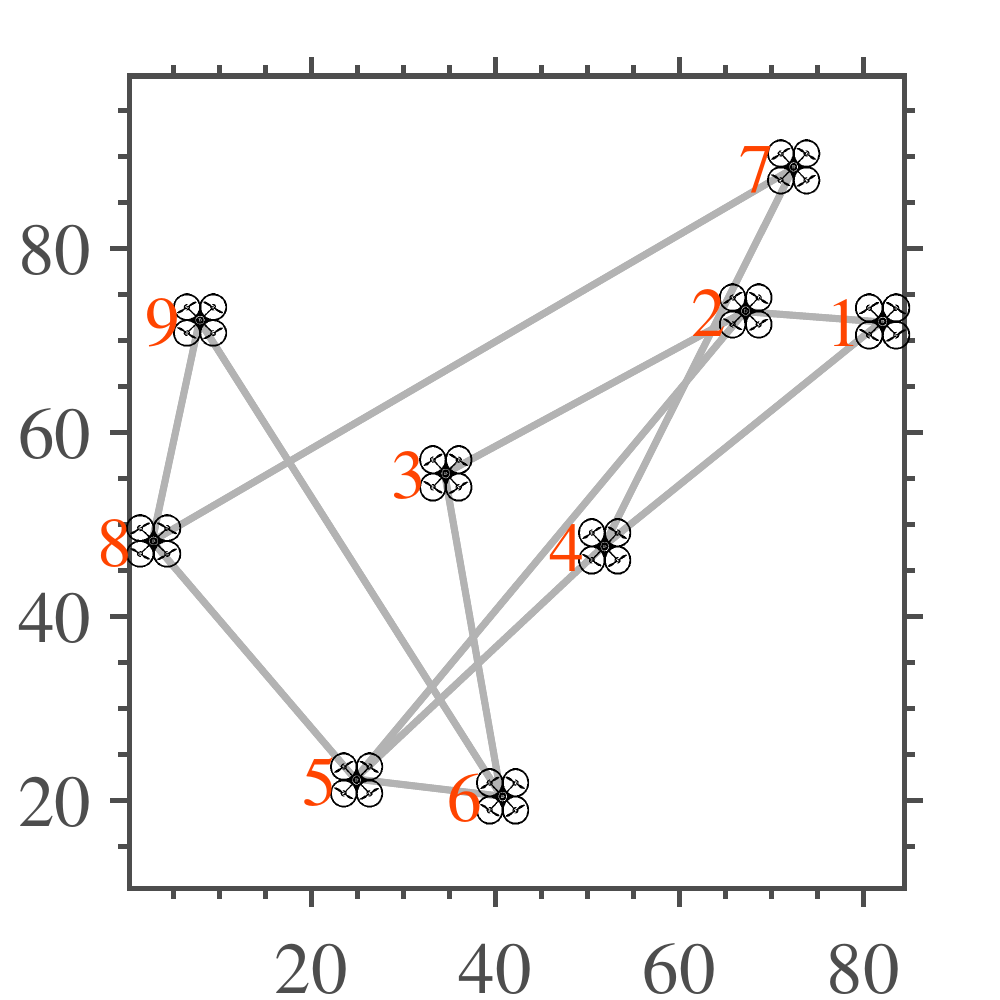}
		\caption{}
	\end{subfigure}%
	\begin{subfigure}[b]{.2\textwidth}
		\includegraphics[trim = 1mm 1mm 1mm 1mm, clip, width=0.97\textwidth] {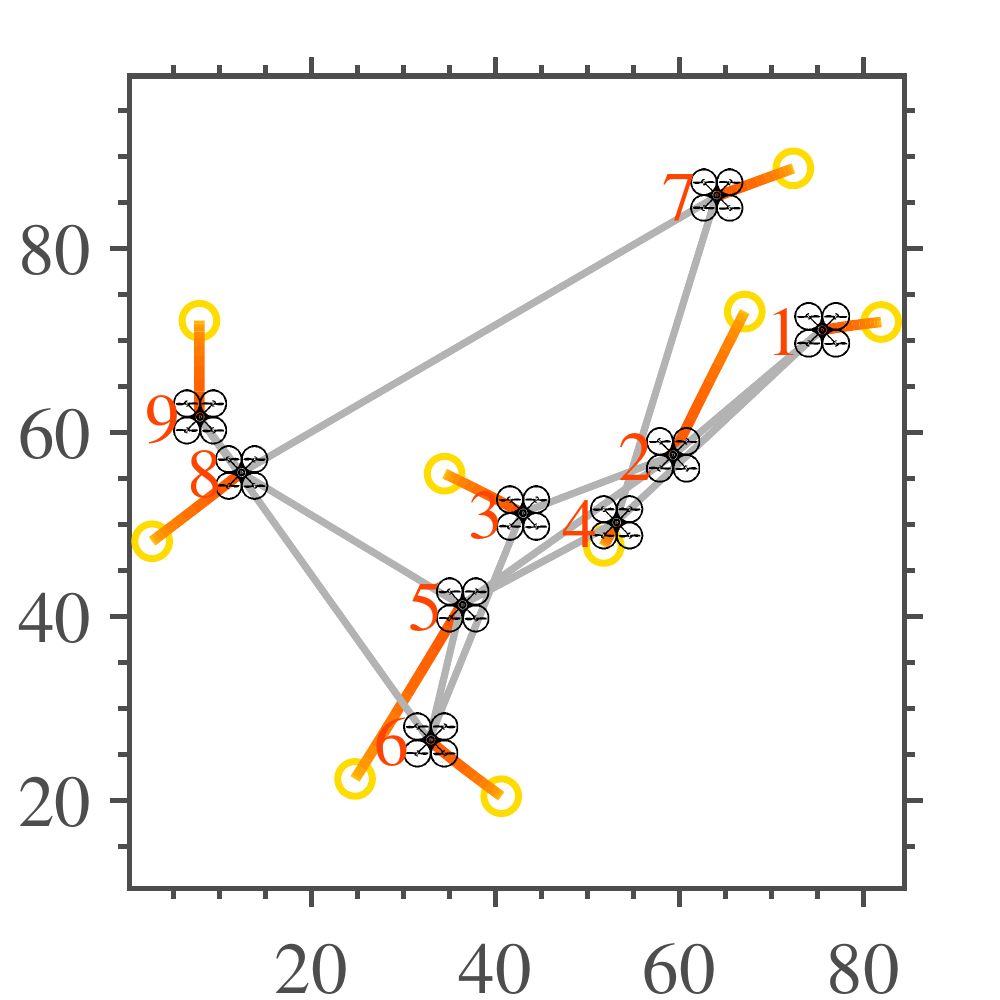}
		\caption{}
	\end{subfigure}%
	\begin{subfigure}[b]{.2\textwidth}
		\includegraphics[trim = 1mm 1mm 1mm 1mm, clip, width=0.97\textwidth] {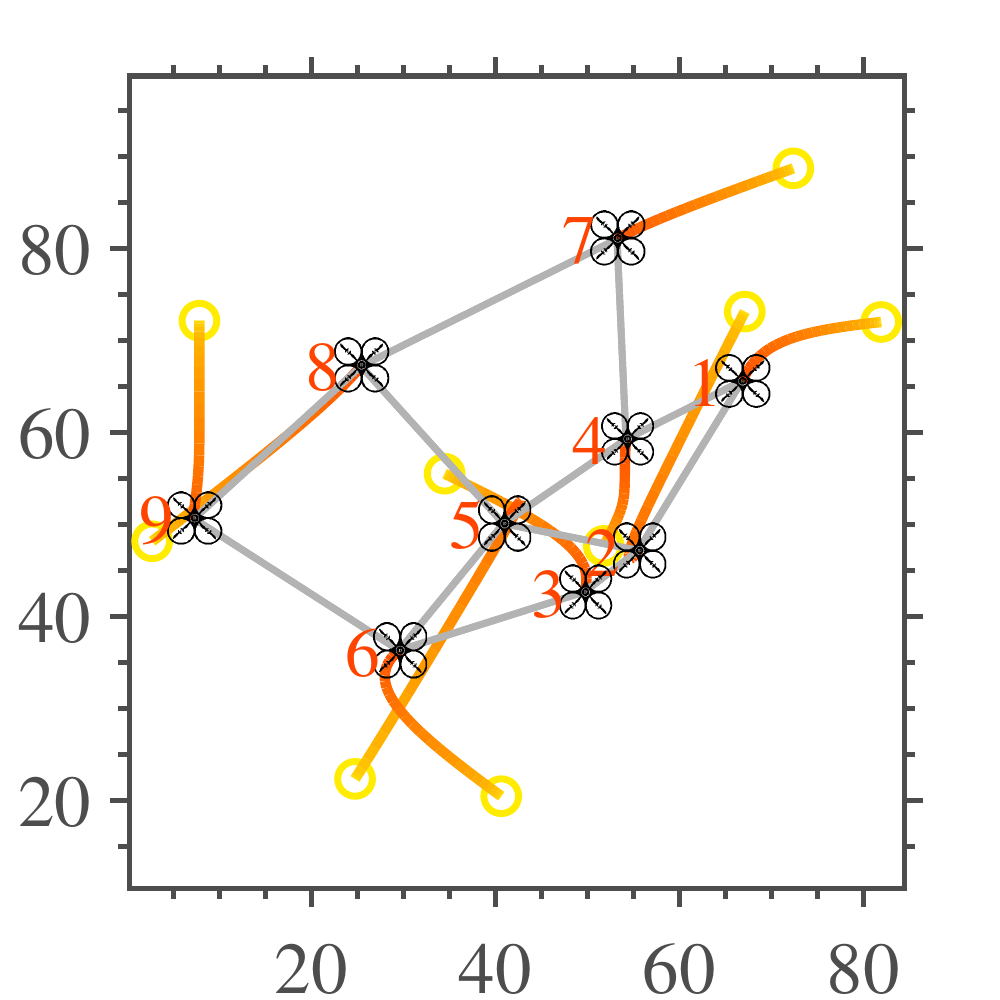}
		\caption{}
	\end{subfigure}%
	\begin{subfigure}[b]{.2\textwidth}
		\includegraphics[trim = 1mm 1mm 1mm 1mm, clip, width=0.97\textwidth] {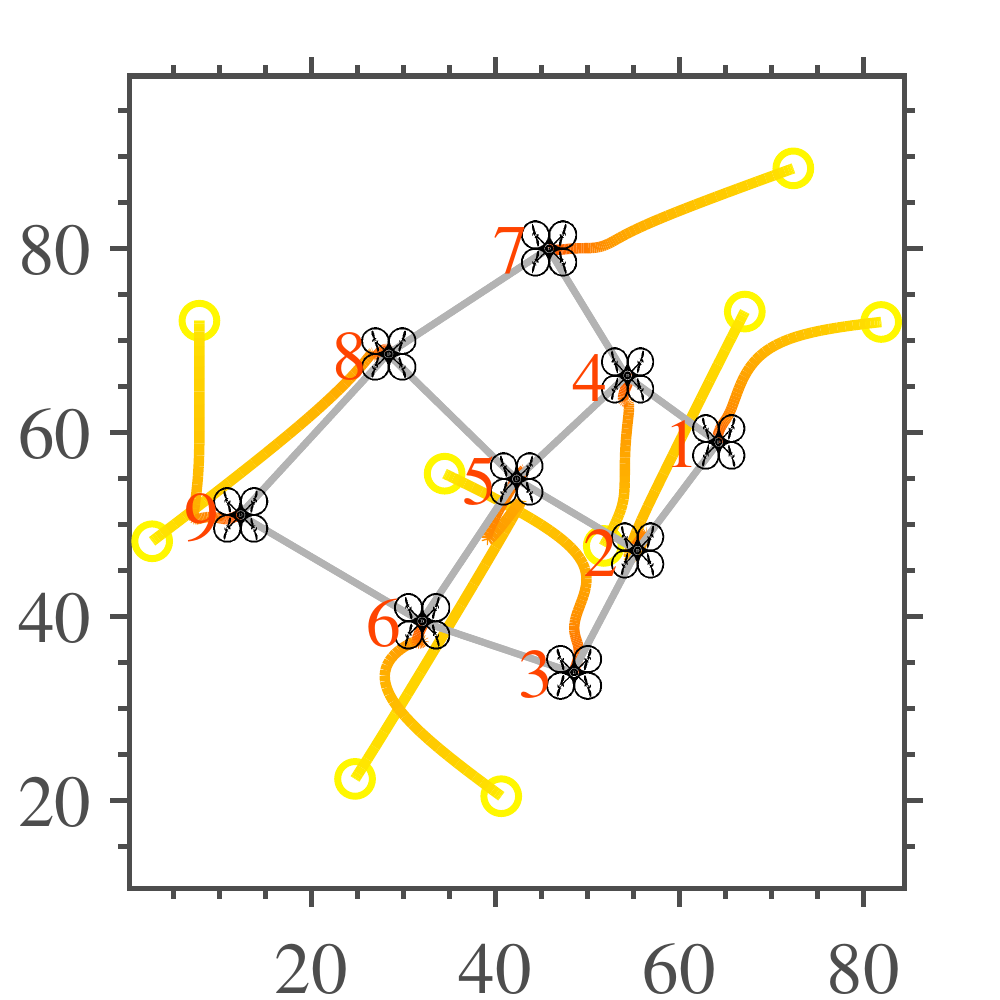}
		\caption{}
	\end{subfigure}%
	\begin{subfigure}[b]{.2\textwidth}
		\includegraphics[trim = 1mm 0mm 1mm 1mm, clip, width=0.97\textwidth] {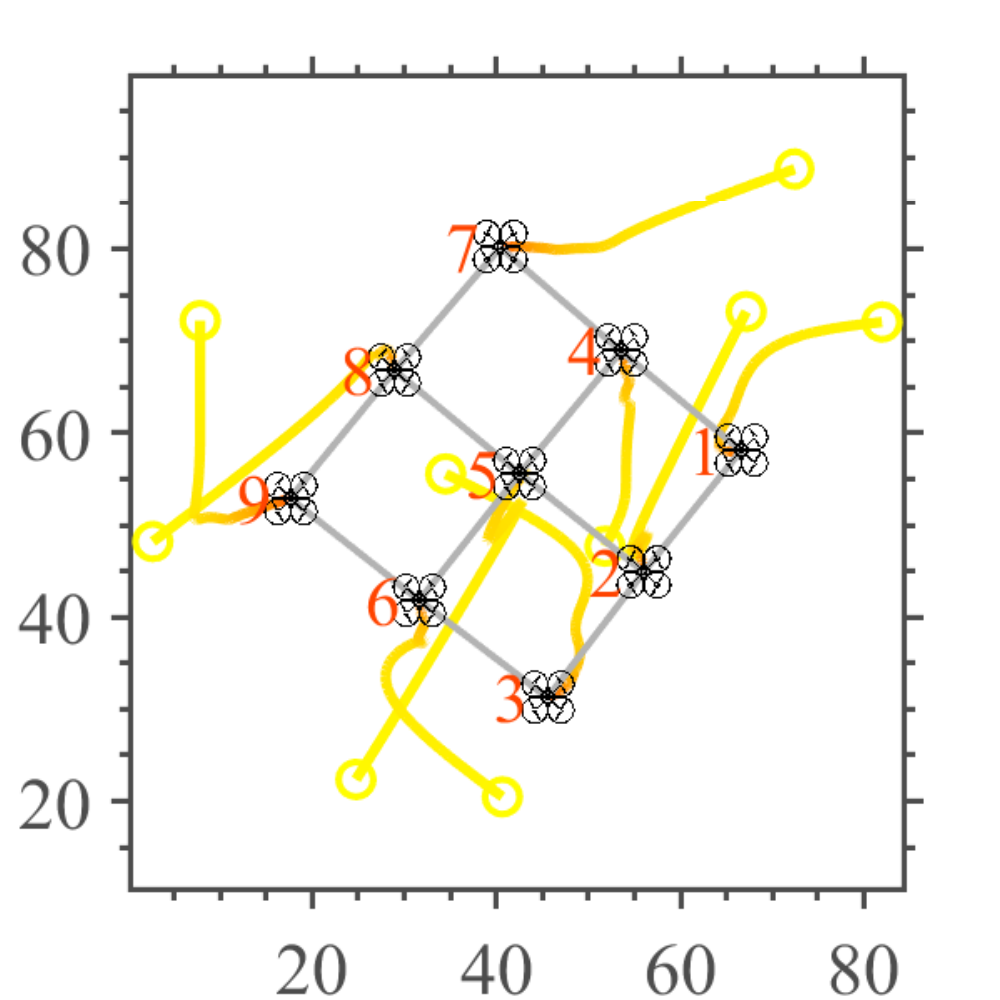}
		\caption{}
	\end{subfigure}%
	\caption{Simulation of 9 quadrotors with a square grid desired formation (actual size of vehicles increased by a factor of 1.5 for better visibility). (a) Top view at $t =  0$s. (b) $t = 4$s. (c) $t = 8$s. (d) $t = 21$s. (e) $t = 80$s.}
	\label{fig:QuadSimul}	
\end{figure*}

\begin{figure*}[]
	\centering
	\begin{subfigure}[b]{.2\textwidth}
		\includegraphics[trim = 1mm 1mm 1mm 1mm, clip, width=0.97\textwidth] {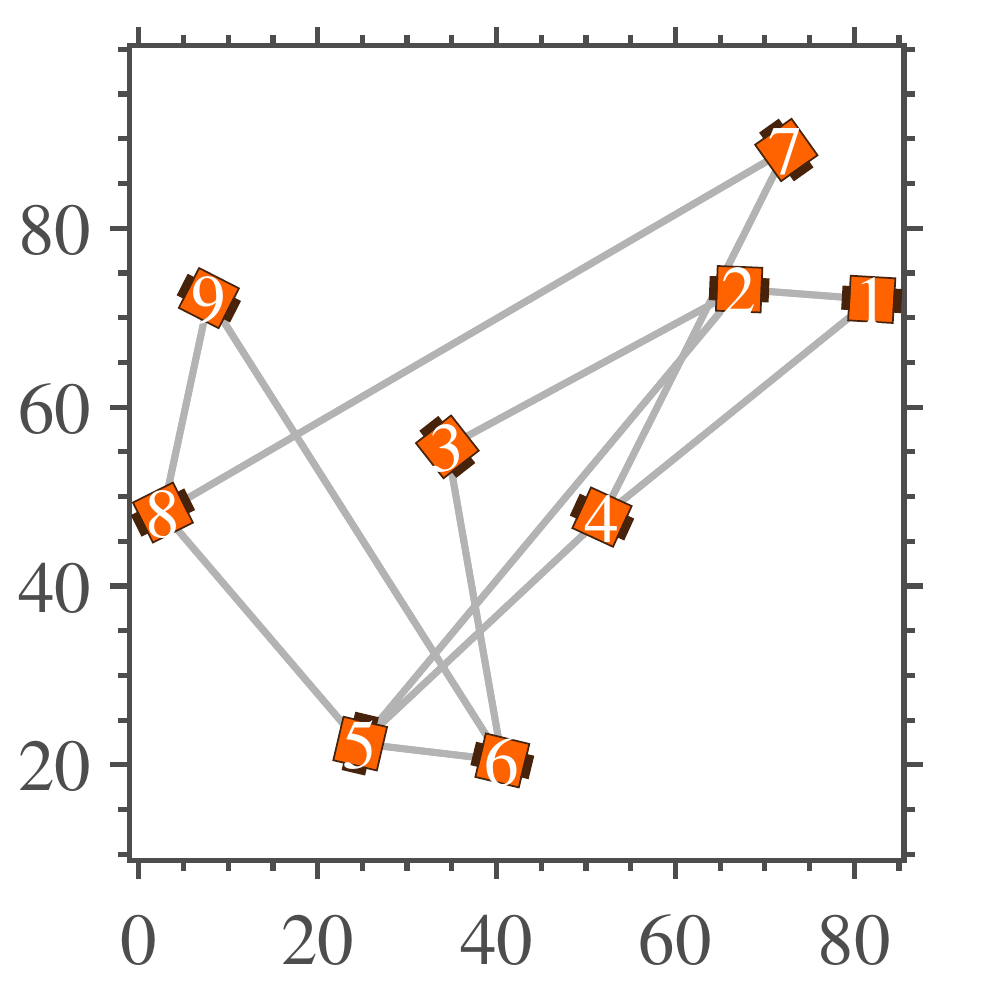}
		\caption{}
	\end{subfigure}%
	\begin{subfigure}[b]{.2\textwidth}
		\includegraphics[trim = 1mm 1mm 1mm 1mm, clip, width=0.97\textwidth] {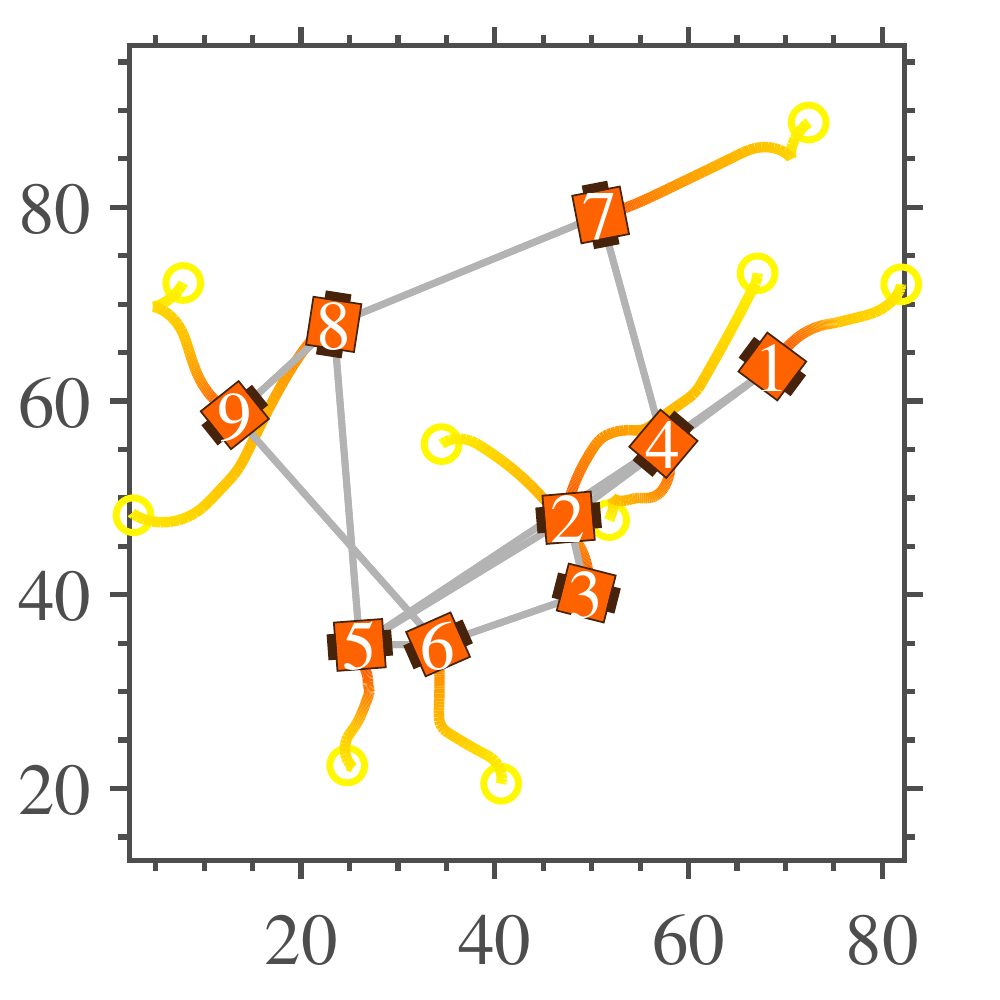}
		\caption{}
	\end{subfigure}%
	\begin{subfigure}[b]{.2\textwidth}
		\includegraphics[trim = 1mm 1mm 1mm 1mm, clip, width=0.97\textwidth] {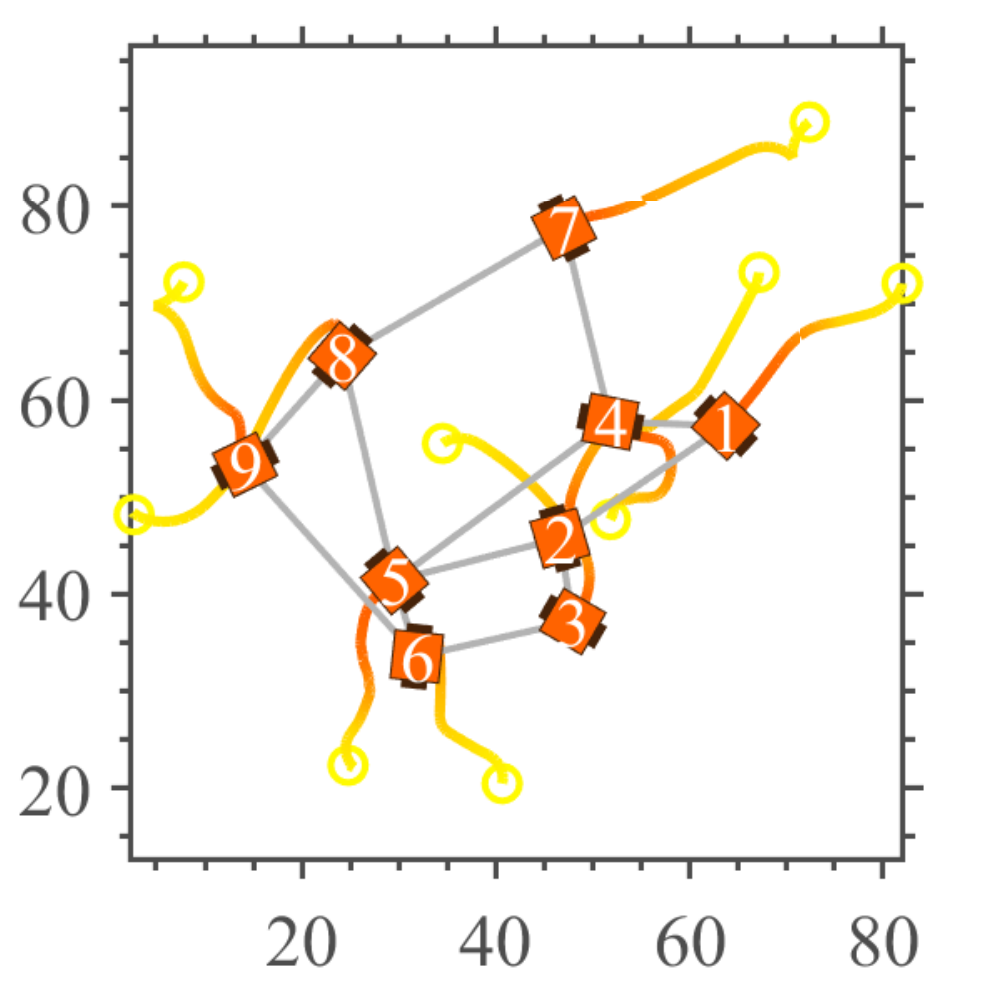}
		\caption{}
	\end{subfigure}%
	\begin{subfigure}[b]{.2\textwidth}
		\includegraphics[trim = 1mm 1mm 1mm 1mm, clip, width=0.97\textwidth] {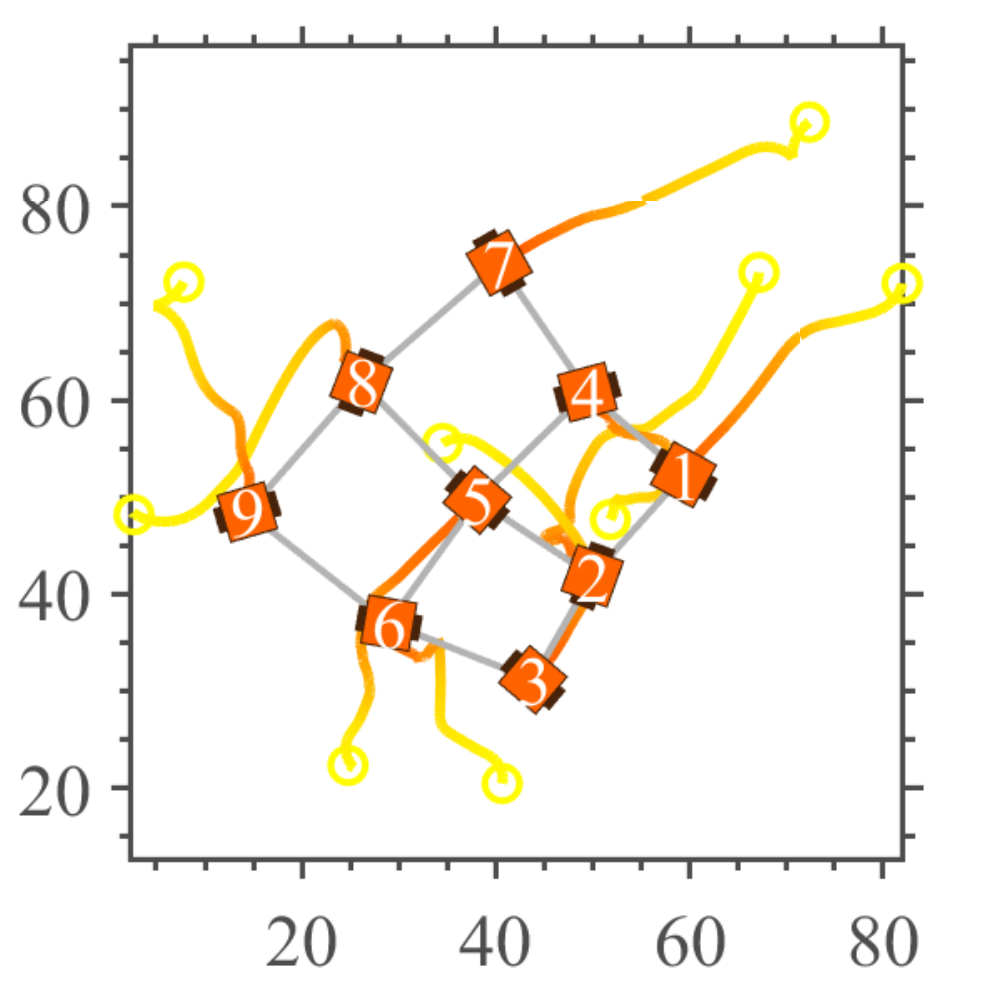}
		\caption{}
	\end{subfigure}%
	\begin{subfigure}[b]{.2\textwidth}
		\includegraphics[trim = 1mm 0mm 1mm 1mm, clip, width=0.97\textwidth] {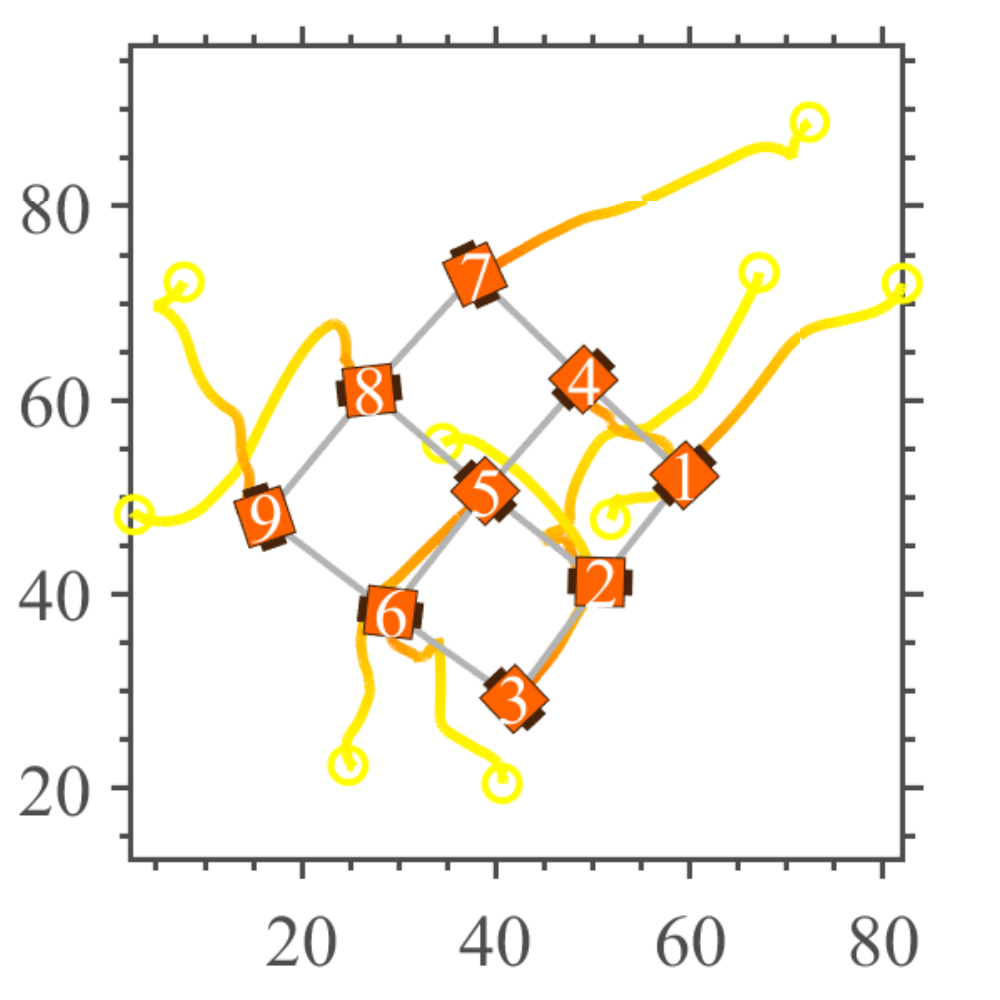}
		\caption{}
	\end{subfigure}%
	\caption{Simulation of 9 unicycles with a square grid desired formation (actual size of vehicles increased by a factor of 1.5 for better visibility). (a) Top view at $t =  0$s. (b) $t = 17$s. (c) $t = 25$s. (d) $t = 45$s. (e) $t = 80$s.}
	\label{fig:UnicycleSimul}	
\end{figure*}

\begin{figure*}[]
	\centering
	\begin{subfigure}[b]{.2\textwidth}
		\includegraphics[trim = 1mm 1mm 1mm 1mm, clip, width=0.97\textwidth] {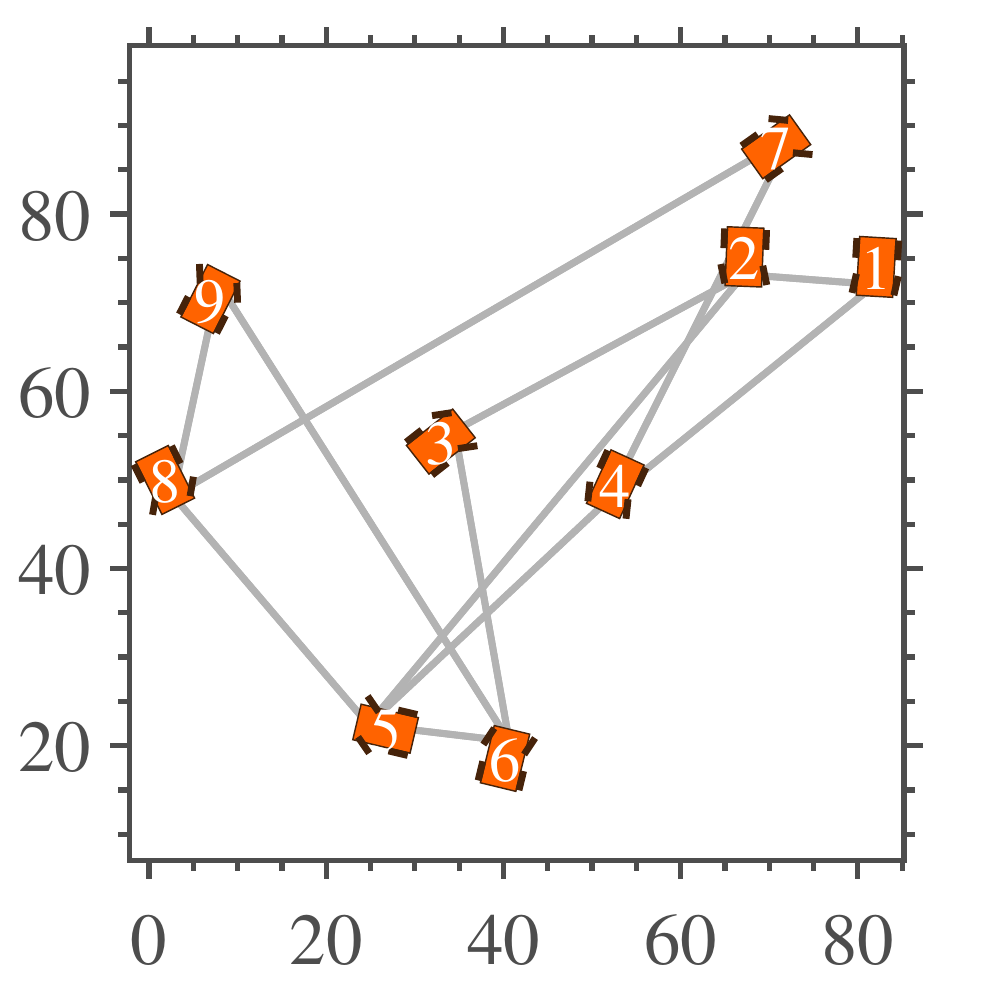}
		\caption{}
	\end{subfigure}%
	\begin{subfigure}[b]{.2\textwidth}
		\includegraphics[trim = 1mm 1mm 1mm 1mm, clip, width=0.97\textwidth] {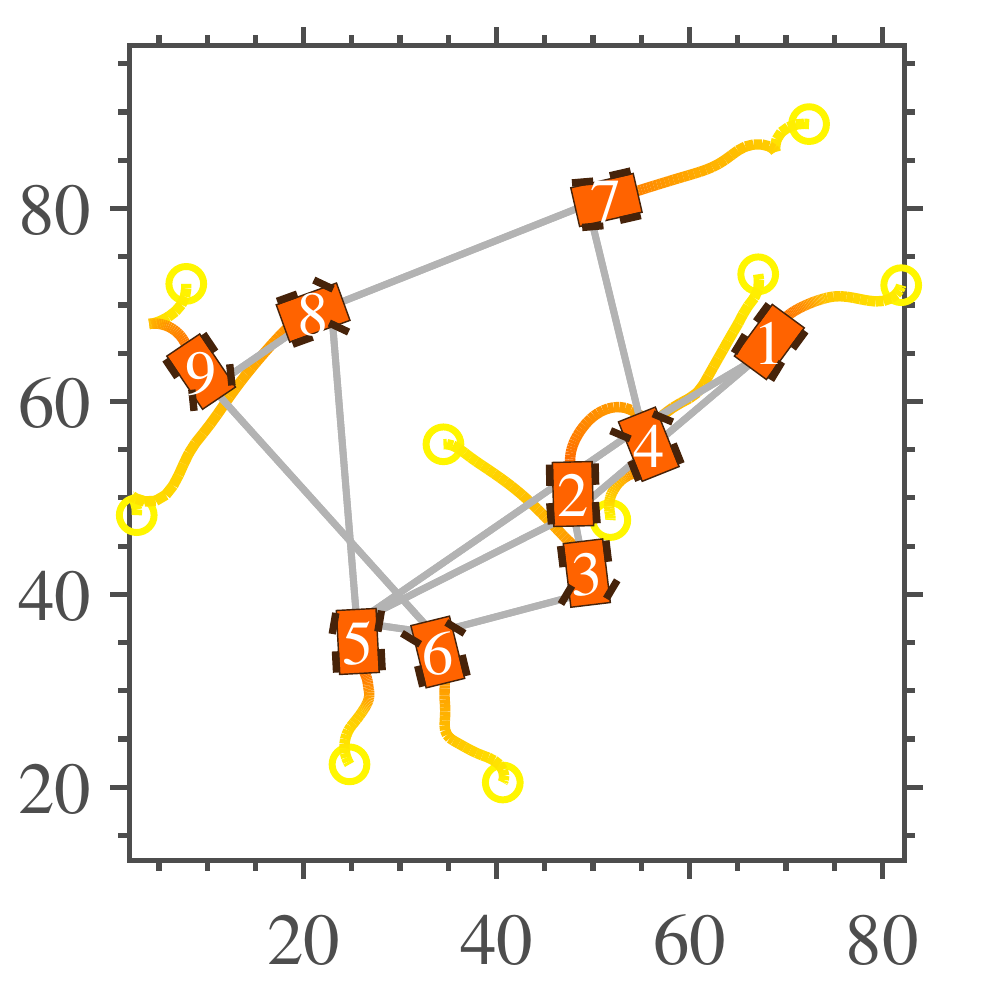}
		\caption{}
	\end{subfigure}%
	\begin{subfigure}[b]{.2\textwidth}
		\includegraphics[trim = 1mm 1mm 1mm 1mm, clip, width=0.97\textwidth] {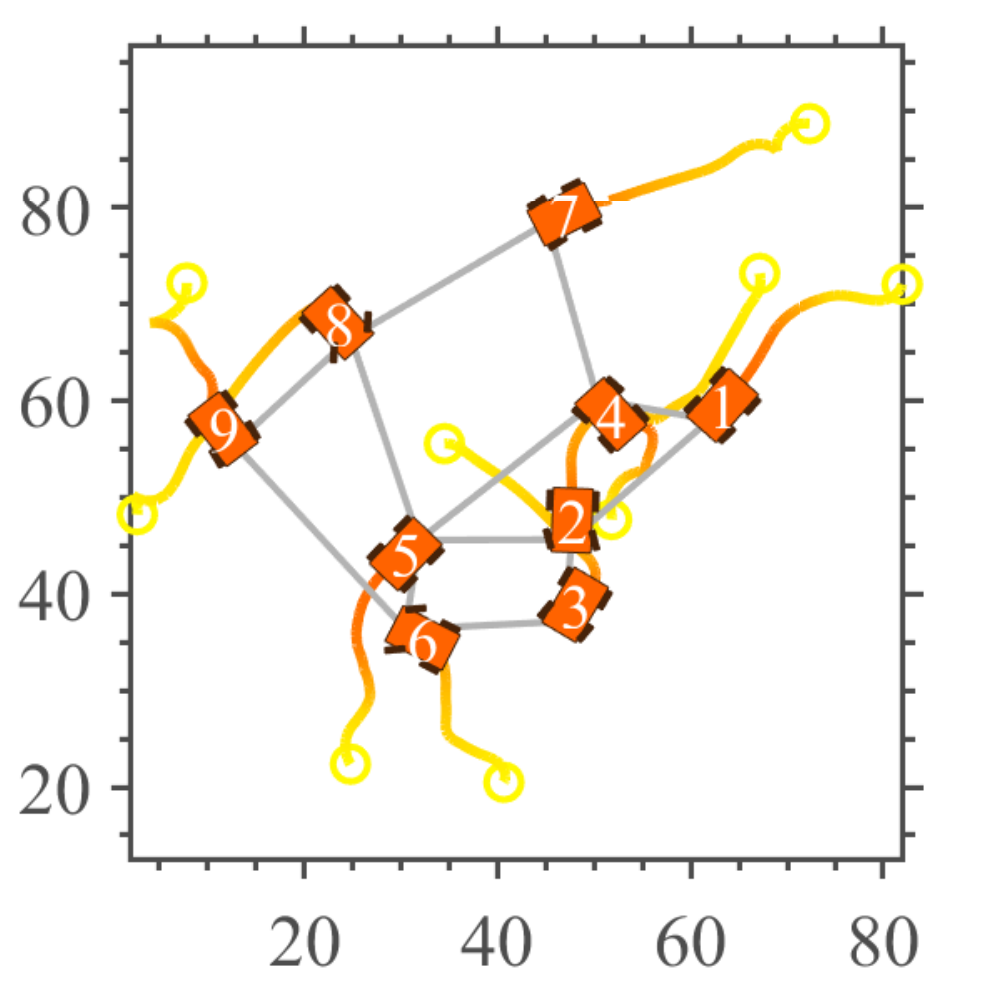}
		\caption{}
	\end{subfigure}%
	\begin{subfigure}[b]{.2\textwidth}
		\includegraphics[trim = 1mm 1mm 1mm 1mm, clip, width=0.97\textwidth] {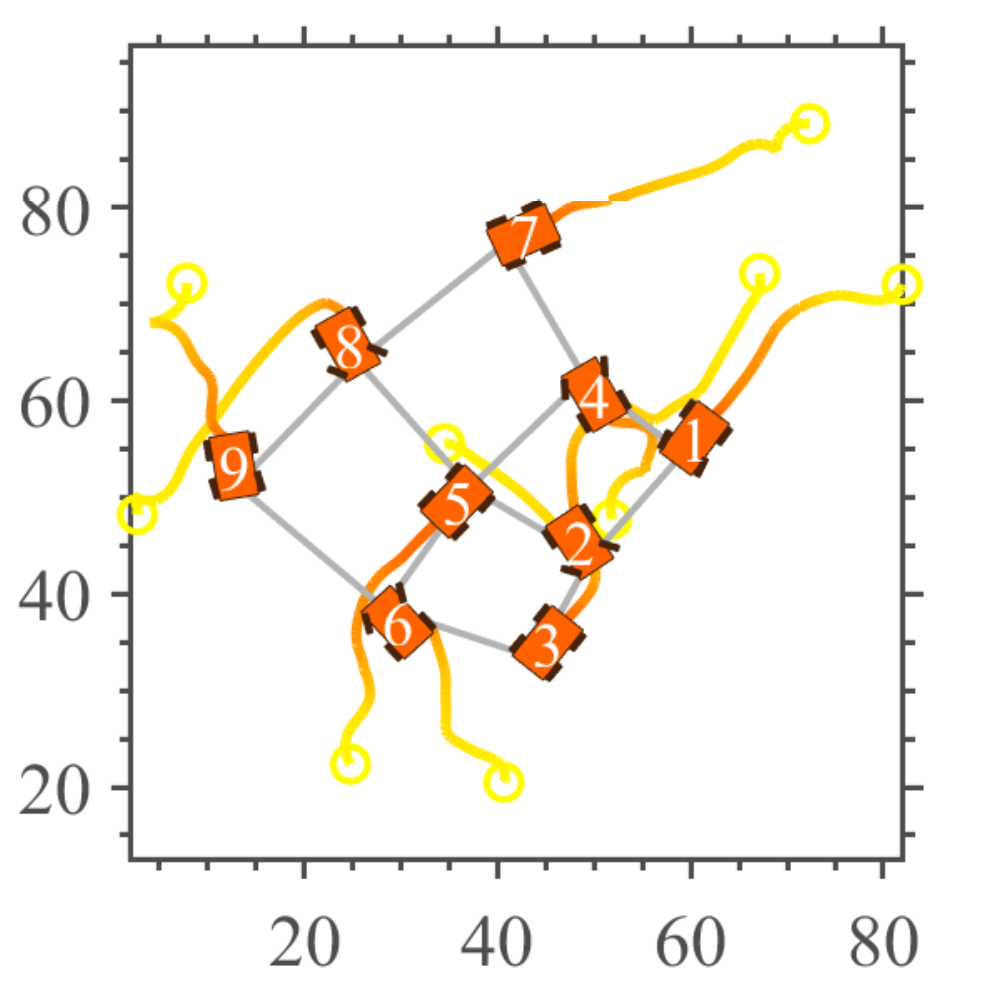}
		\caption{}
	\end{subfigure}%
	\begin{subfigure}[b]{.2\textwidth}
		\includegraphics[trim = 1mm 0mm 1mm 1mm, clip, width=0.97\textwidth] {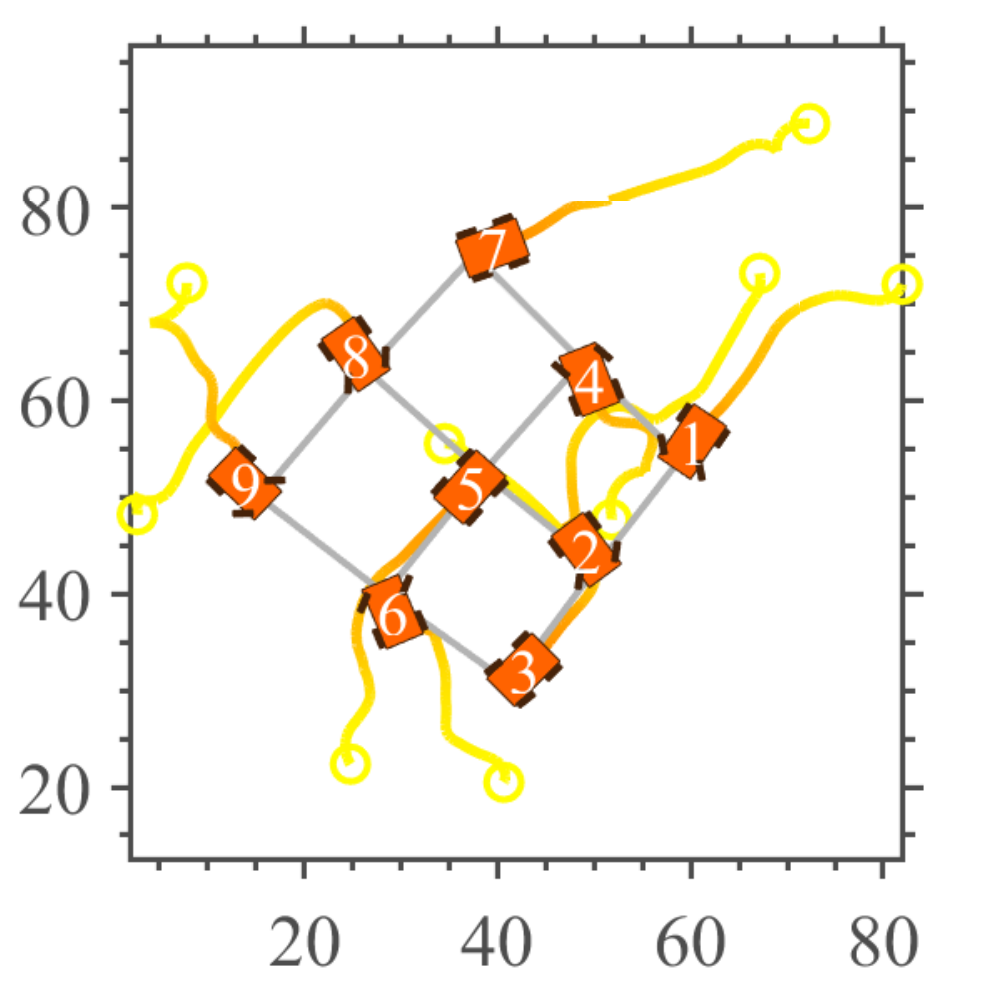}
		\caption{}
	\end{subfigure}%
	\caption{Simulation of 9 cars with a square grid desired formation (actual size of vehicles increased by a factor of 1.5 for better visibility). (a) Top view at $t =  0$s. (b) $t = 18$s. (c) $t = 29$s. (d) $t = 43$s. (e) $t = 80$s.}
	\label{fig:CarSimul}	
\end{figure*}

\subsection{Quadrotors}

Based on the quadrotor dynamics described in Example~\ref{ex:Quadrotor}, a simulation with 9 quadrotors and a  scale-free square grid desired formation is performed. Although the control design is based on the linearized dynamics about the quadrotor's hover point, the original nonlinear quadrotor dynamics given in \eqref{eq:QuadNonlinDynam} is used for the simulation.
To demonstrate robustness to switches in the inter-agent sensing topology, 
the sensing graph is switched among the topologies illustrated in Fig.~\ref{fig:Graphs} based on a randomly generated switching signal shown in the figure. 
\phantomsection \label{p:b13} We further performed simulations in which the topology changes are based on the robots' proximity. Since performance was similar to the results presented here, we do not report the results, however, they can be viewed in the supplemental video available at {\color{black} \href{https://youtu.be/3IcikoWBZJE}{https://youtu.be/3IcikoWBZJE}}.
The control gains associated with the desired formation are computed from Algorithm \ref{alg:GainDesign}, where we used \eqref{eq:OptimCVXJoint} to obtain gains that jointly stabilize all  topologies. The nonzero eigenvalues of computed $A \in \br^{18 \times 18}$ matrices range from $-0.035$ to $-0.497$. The control law used for each quadrotor is chosen according to \eqref{eq:HigherOrderControl2}, where gains are set as $k_0 = 2,\, k_1 = 2,\,  k_2 = 3,\, k_3 = 3$ to make the closed-loop state matrix $\bar{A}$ stable for all topologies. Using these gains, the real part of nonzero eigenvalues of $\bar{A}$ matrices range from $-0.038$ to $-2.0$.
To avoid collision among quadrotors, the distributed collision avoidance strategy in Algorithm \ref{alg:ColAvoidance} with $d_c = 8$ and $r = 4$ units of length is employed.

Fig.~\ref{fig:QuadSimul}(a)-(e) shows the top view of quadrotors at different time instances.  The sensing graph among agents is shown by gray lines connecting the quadrotors. This sensing graph switches throughout the simulation according to Fig.~\ref{fig:Graphs}.
The initial positions of the quadrotors are chosen randomly, and are shown in Fig.~\ref{fig:QuadSimul}(a). As can be seen in Figs.~\ref{fig:QuadSimul}(b)-(e), the proposed control strategy brings the agents to the desired formation.  
Note that when the distance between two quadrotors becomes less than 8 units of length, the collision avoidance strategy is engaged to rotate the control direction outside of the collision cone. Consequently, none of the quadrotors collide during the simulation. 
Further notice that since the control only uses the local relative position measurements, the desired formation is achieved up to a rotation and translation. That is, the orientation of the square formation is not controlled.

\phantomsection \label{p:b14} 
We point out that in the quadrotor Example~\ref{ex:Quadrotor}, the inputs are $u^x,\, u^y$ and the outputs are the $x$-$y$ positions (since we are concerned with planar formations). The input $u^a$ affects the $x$-$y$ positions through $R$ due to the coupled dynamics, and the zero dynamics consists of the state variables $z,\, \psi,\, \omega_z$, which are unobservable from the outputs, however, are asymptotically stable. The theoretical convergence guarantees of the proposed control are based on the assumption of input-to-state feedback linearizability. Nonetheless, as can be seen from the simulation results which are based on the original nonlinear dynamics, the quadrotors achieve the desired formation. This suggests potential applicability of the proposed control to systems with asymptotically stable zero dynamics, which can be expected due to the robustness properties.

\subsection{Unicycles}

The control strategy \eqref{eq:kinemCtrl} for agents with unicycle dynamics  is considered in a simulation with 9 unicycles and a square grid desired formation.
The unicycle dynamics \eqref{eq:DynamUnicycle} are used to test the performance of control in the presence of unmodeled dynamics, where values of parameters $a,\, b,\, c,\, d$ are chosen randomly for each agent with uniform distribution in the interval $[5,\, 10]$. 
All linear and angular velocities are saturated by the maximum allowed velocities of $v_{\max} = 3$ units of length per second and $\omega_{\max} = \pi / 4$ radians per second, respectively. 
The control gain matrices designed for quadrotors in the previous section are used for unicycle agents, showing that the control gains found from Algorithm \ref{alg:GainDesign} can be used for vehicles with a variety of dynamics to achieve the same desired formation.

To allow a better comparison between trajectories of agents with different dynamics, the unicycle agents start from the same initial condition as  quadrotors, as can be seen in Fig.~\ref{fig:UnicycleSimul}(a), and the sensing topology among them switches according to Fig.~\ref{fig:Graphs}.
The position of agents at other time instances are shown in  Figs.~\ref{fig:UnicycleSimul}(b)-(e), where by using the collision avoidance strategy in Algorithm \ref{alg:ColAvoidance} with $d_c = 8$ and $r = 4$ units of length, no collisions occur as the unicycles converge to the desired formation. 
Similar to quadrotors, the desired formation is scale-free and achieved up to a rotation and translation with respect to the global coordinate frame that is unknown to agents.

\subsection{Cars}

The control strategy \eqref{eq:CarKinemControl} for agents with front-wheel drive car dynamics is considered in a simulation with 9 cars and a square grid desired formation.
The car dynamics \eqref{eq:DynamCarFrontOriginal} is used to test the performance of control in the presence of unmodeled dynamics, 
where values of parameters $a,\, b,\, c,\, d$ are chosen as the same values for unicycle agents to allow a better comparison. 
The driving and steering velocities of cars are saturated by the maximum allowed velocities of $v_{\max} = 3$ units of lengths per second and $\omega_{\max} = \pi / 4$ radians per second, respectively. Furthermore, all steering angles are confined to the interval of $[-\pi / 4, \, \pi / 4]$ radians to model the practical bounds on the steering angle of wheels in cars.
The control gain matrices used for quadrotors and unicycles are used in the simulation.

The sensing topology switches according to Fig.~\ref{fig:Graphs}, and the initial position of cars is shown in Fig.~\ref{fig:CarSimul}(a), which is the same as quadrotors and unicycles to allow a better comparison.
The position of cars at other instances of time are shown in  Figs.~\ref{fig:CarSimul}(b)-(e), where by using the collision avoidance strategy with $d_c = 8$ and $r = 4$ units of length no collisions occur as the cars converge to the desired formation. 
Note that the attained square grid formation is with respect to the front axle's center of each car, i.e., the origin of car's local coordinate frame in Fig.~\ref{fig:CarModel}. Furthermore, the heading of the cars at the final formation is not specified and can take an arbitrary value.

\section{Experimental Results} \label{sec:Experiments}

In this section, we validate the proposed control strategies experimentally on a distributed multi-robot platform. 
Our experimental study is limited to the cases of single-integrator and unicycle dynamics, as we do not  have a fleet of autonomous cars. 
Links to the implementation code and technical details is provided in the Supplementary Material section.

\subsection{Experimental Platform}

\begin{figure} 
	\begin{center}	
		\includegraphics[trim = 75mm 82mm 65mm 72mm, clip, width=0.49\textwidth] {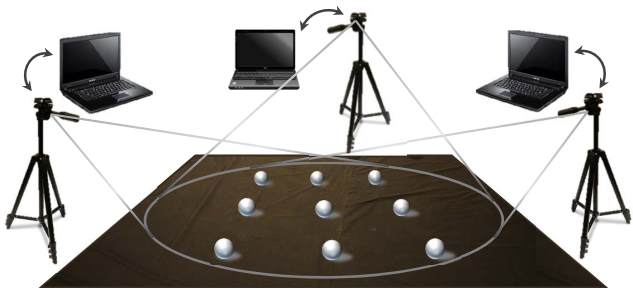}
		\caption{Schematics of the experimental setup.}	
		\label{fig:ExprSetup}			
	\end{center}
\end{figure}

\begin{figure} 
	\begin{center}	
		\includegraphics[trim = 0mm 10mm 0mm 0mm, clip, width=0.13\textwidth] {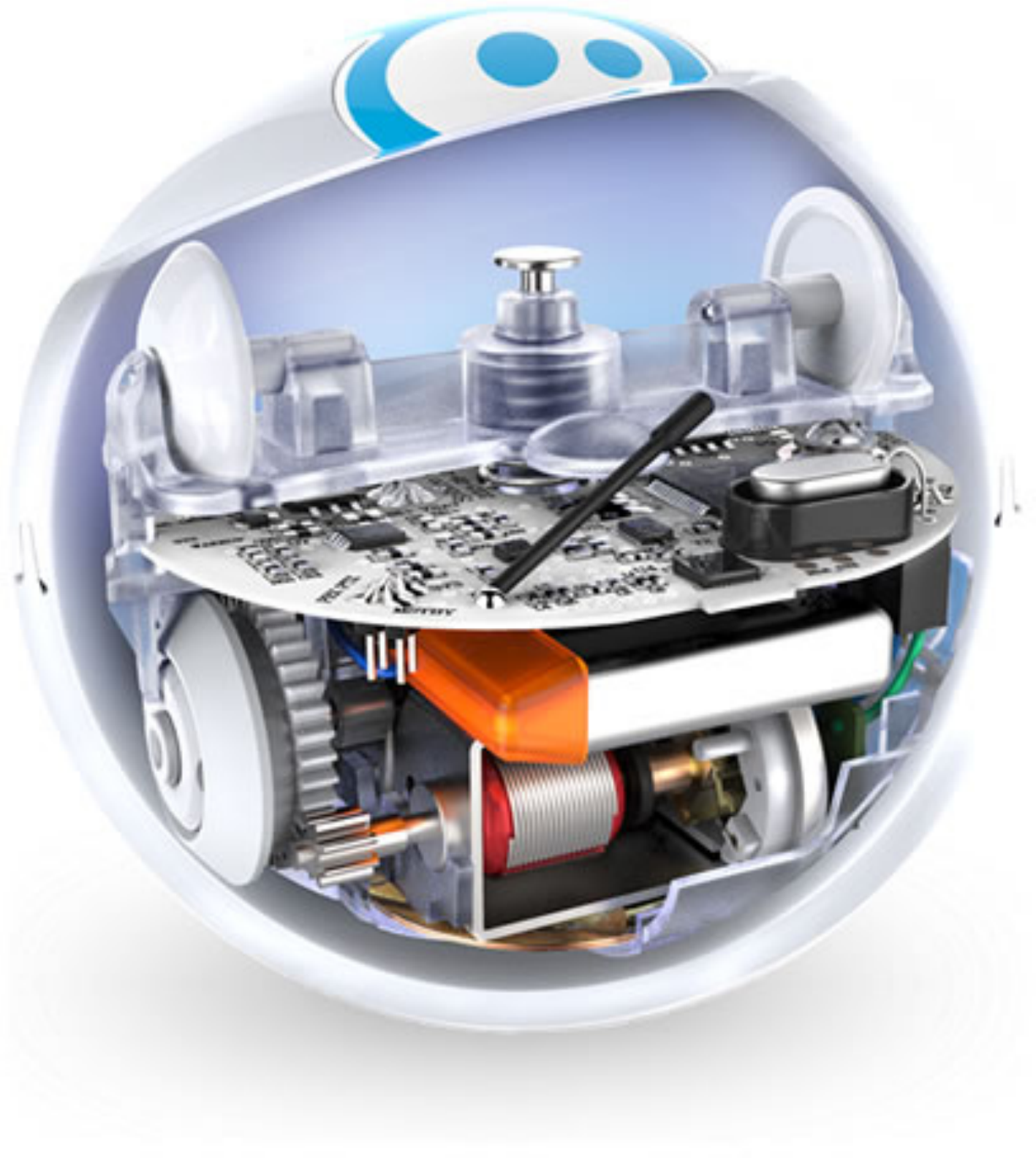}
		\caption{Schematics of a Sphero robot.}	
		\label{fig:Sphero}			
	\end{center}
\end{figure}

Our experimental platform consists of the Sphero 2.0 robots, laptop computers with Bluetooth adapters to control the robots, and Logitech C920 webcams to provide vision feedback. 
As illustrated in Fig.~\ref{fig:ExprSetup}, a group of Sphero robots are placed in an arena that is overseen by webcams. The video stream provided by each webcam is used in an image processing script to detect and track the Spheros via blob detection \cite[Sec.13.1]{Corke2017}. 
The coordinates of each robot are estimated by mapping the pixel position of the robot in the image to the $x$-$y$ Euclidean coordinates on the arena floor. 
This is done by initially placing a checkerboard at an arbitrary location on the floor and using PnP algorithm \cite{Lepetit2009} to estimate the relative orientation of the ground plane in the camera's coordinate frame.
The coordinates are then given by finding the intersection of the ray through the robot image and the ground plane generated by PnP.

The estimated coordinates of Spheros are used by each computer to calculate the control according to the specified distributed formation control strategy.
The desired control action is then communicated to each robot over Bluetooth.
The experimental setup is distributed in the sense that each computer is responsible for controlling a subset of robots, and computers do not communicate during the experiment. 
Furthermore, the control computed by each computer respects the sensing graph specified by the user and does not use any additional information that may be available.

The schematics of a Sphero robot are shown in Fig.~\ref{fig:Sphero}. The robot consists of a differential-wheeled internal platform that is enclosed in a spherical shell. Rotation of the internal wheels induces a roll motion of the outer shell.
To test the control strategy proposed for single-integrator agents, a low-level PID controller is employed to first orient the internal platform along the desired direction, and then roll the robot forward at the desired speed. For the unicycle agents, the low-level PID controller adjusts the wheel velocities such that the internal differential drive platform have the desired angular and linear velocities.

\subsection{Triangle Formation}

\begin{figure}[t]
	\centering
	\begin{subfigure}[b]{.5\linewidth}
		\includegraphics[trim = 1mm 1mm 1mm 1mm, clip, width=0.97\textwidth] {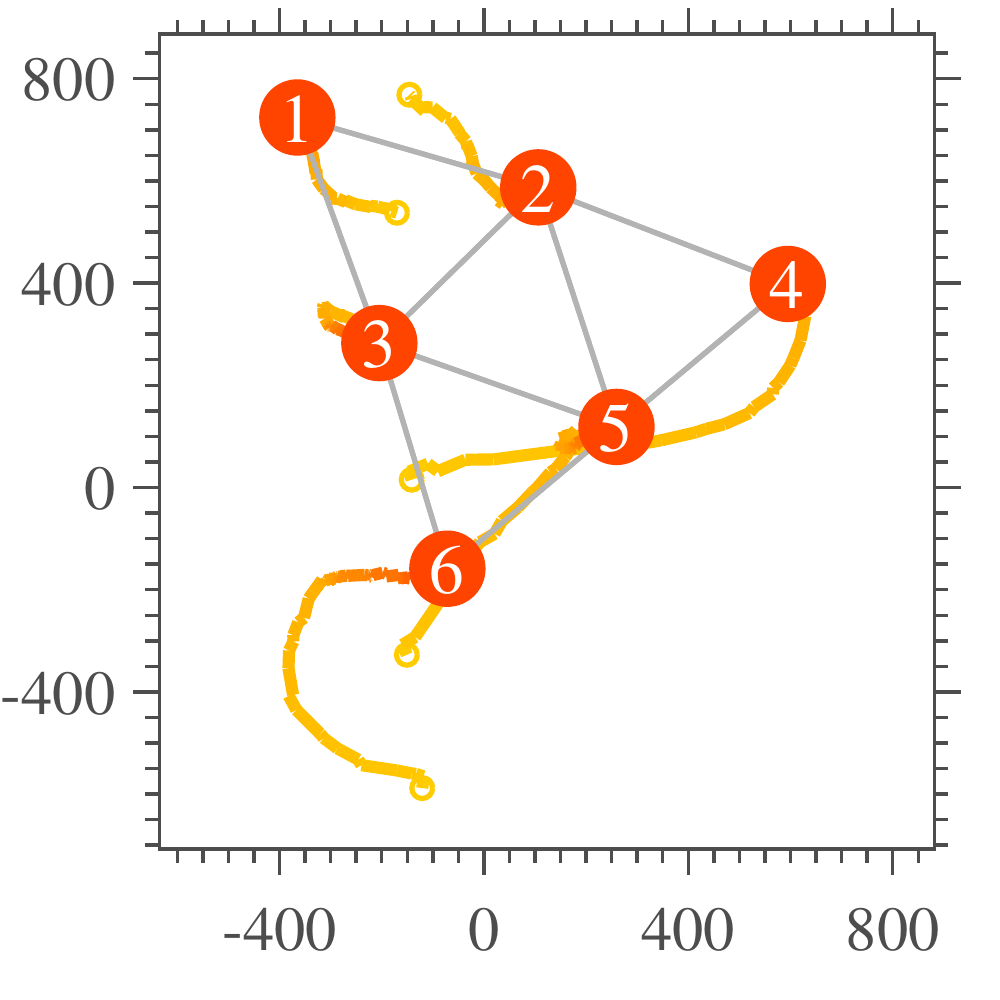}
		\caption{}
	\end{subfigure}%
	\begin{subfigure}[b]{.5\linewidth}
		\includegraphics[trim = 1mm 1mm 1mm 1mm, clip, width=0.97\textwidth] {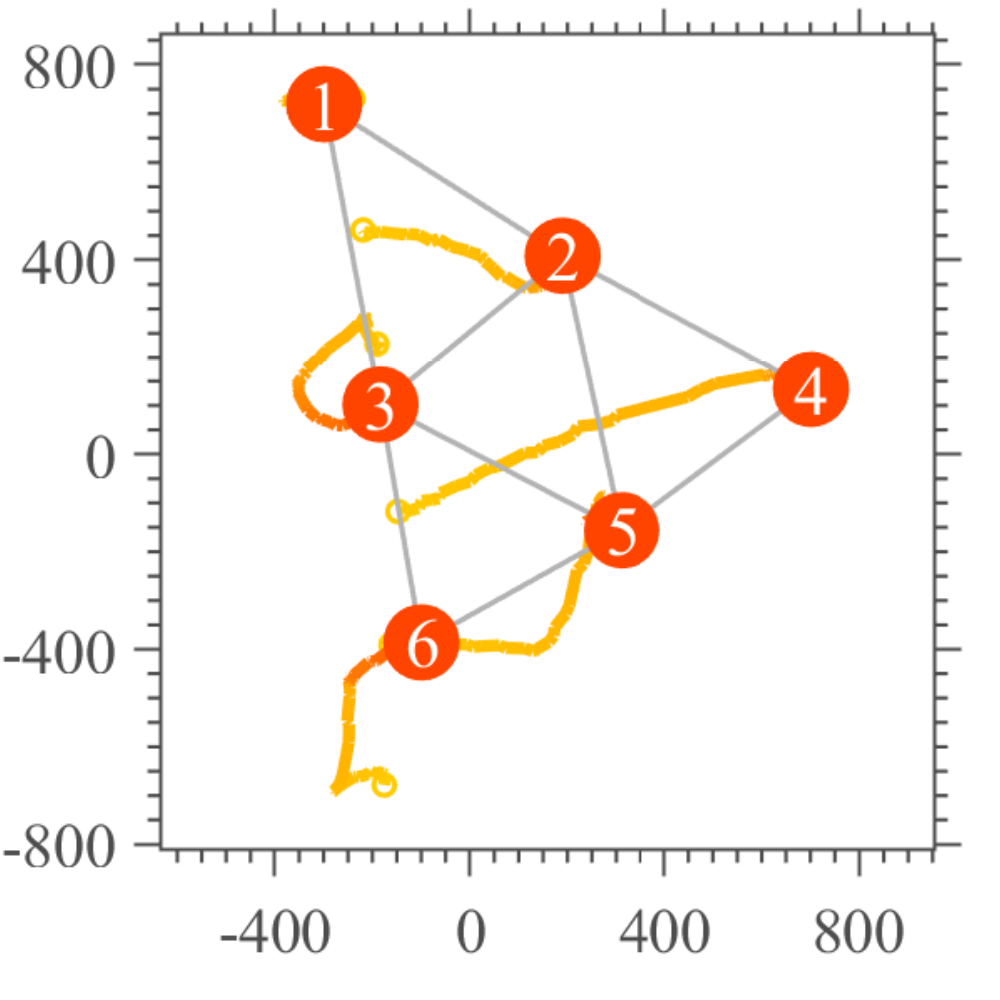}
		\caption{}
	\end{subfigure}%
	\caption{Trajectory of robots estimated from webcam images (a) under single-integrator control, (b) under unicycle control. Units are in millimeter.}
	\label{fig:Triangle_Reconst}	
\end{figure}

\begin{figure}[t]
	\centering
	\begin{subfigure}[b]{.5\linewidth}
		\includegraphics[trim = 1mm 1mm 1mm 1mm, clip, width=0.97\textwidth] {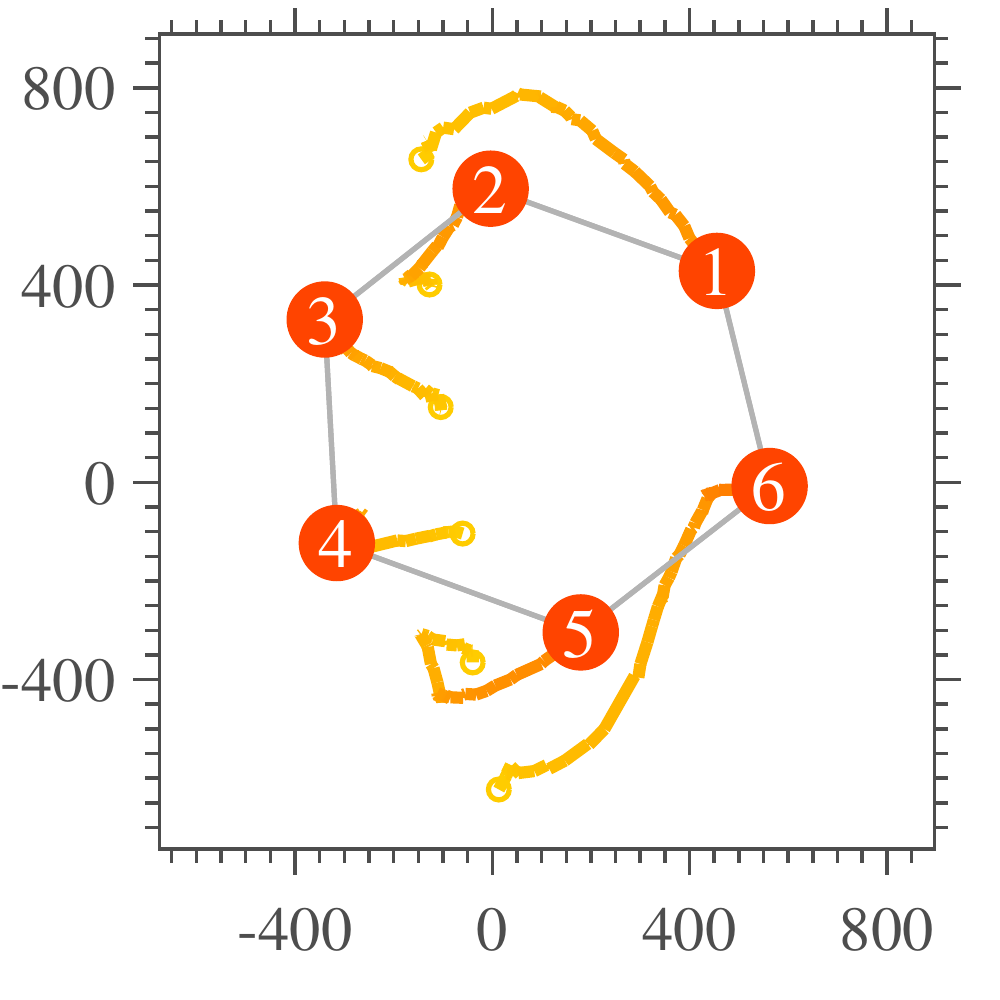}
		\caption{}
	\end{subfigure}%
	\begin{subfigure}[b]{.5\linewidth}
		\includegraphics[trim = 1mm 1mm 1mm 1mm, clip, width=0.97\textwidth] {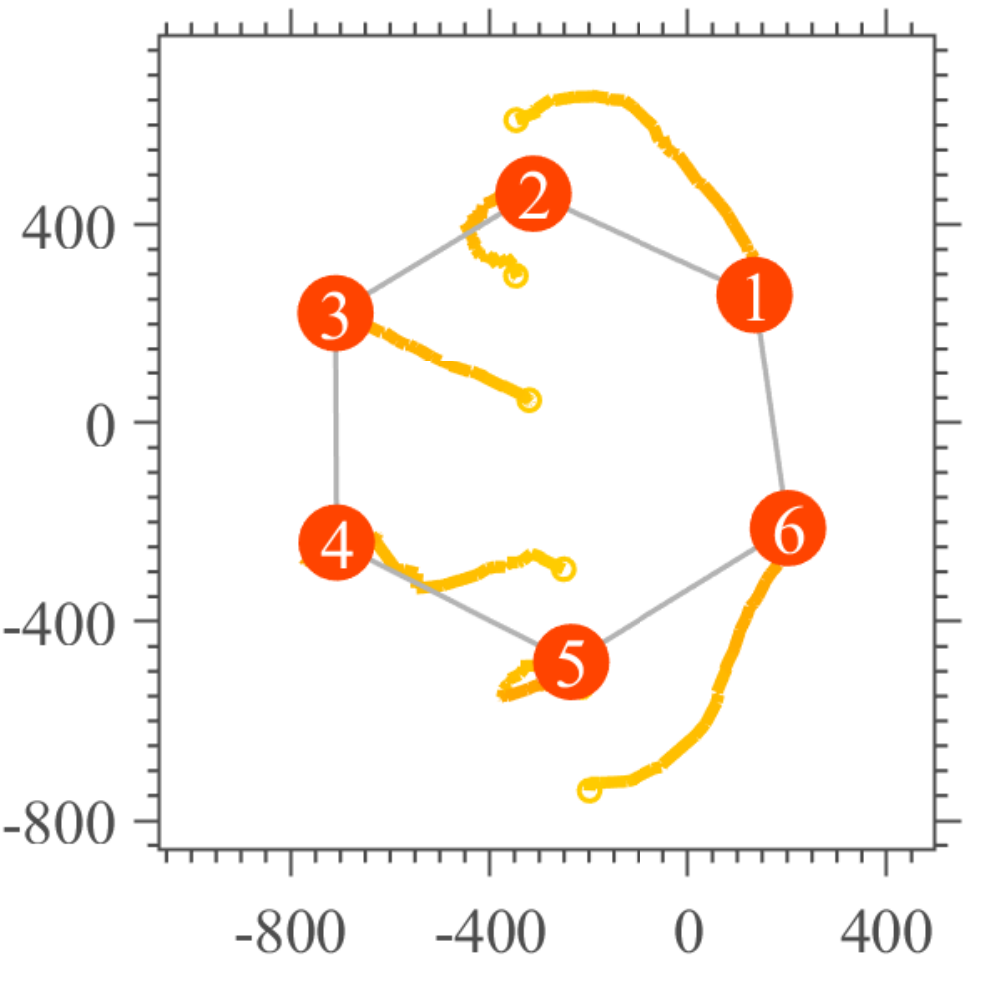}
		\caption{}
	\end{subfigure}%
	\caption{Trajectory of robots estimated from webcam images (a) under single-integrator control, (b) under unicycle control. Units are in millimeter.}
	\label{fig:Hexagon_Reconst}	
\end{figure}

\begin{figure}[t]
	\centering
	\begin{subfigure}[b]{.5\linewidth}
		\includegraphics[trim = 1mm 1mm 1mm 1mm, clip, width=0.97\textwidth] {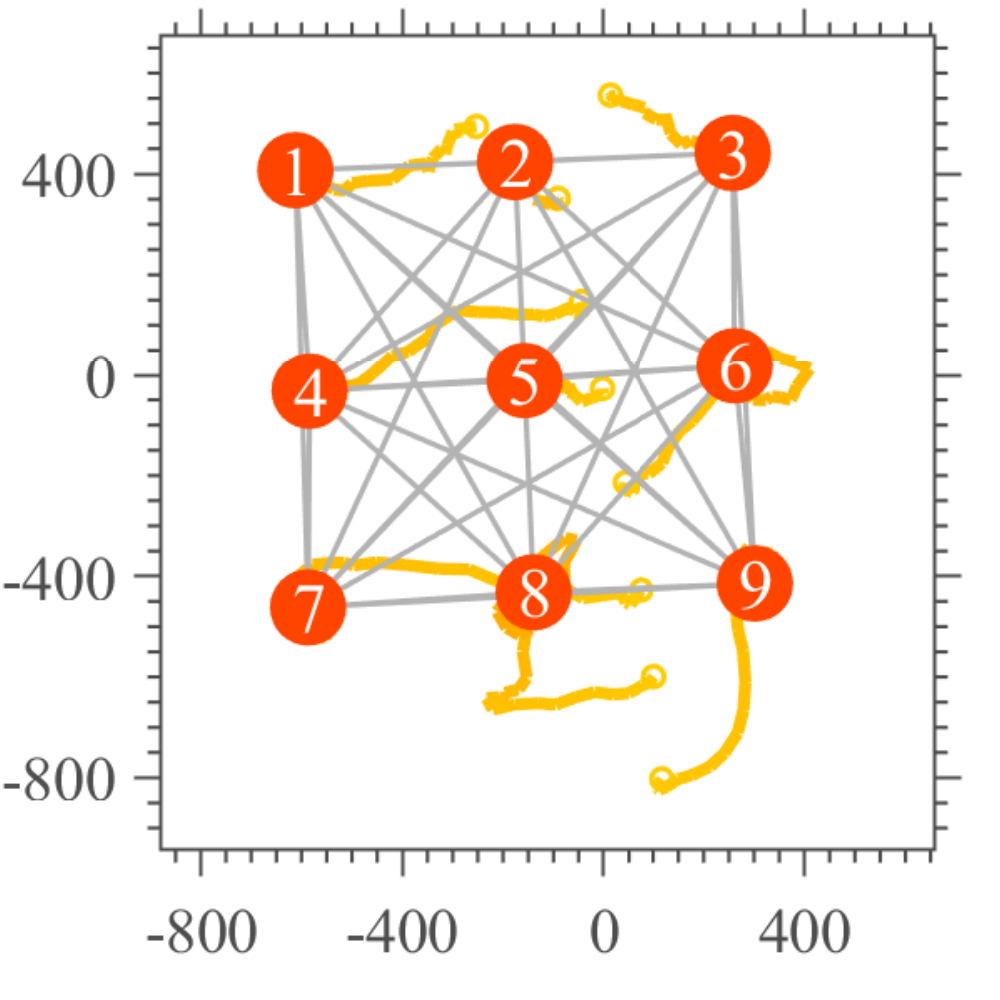}
		\caption{}
	\end{subfigure}%
	\begin{subfigure}[b]{.5\linewidth}
		\includegraphics[trim = 1mm 1mm 1mm 1mm, clip, width=0.97\textwidth] {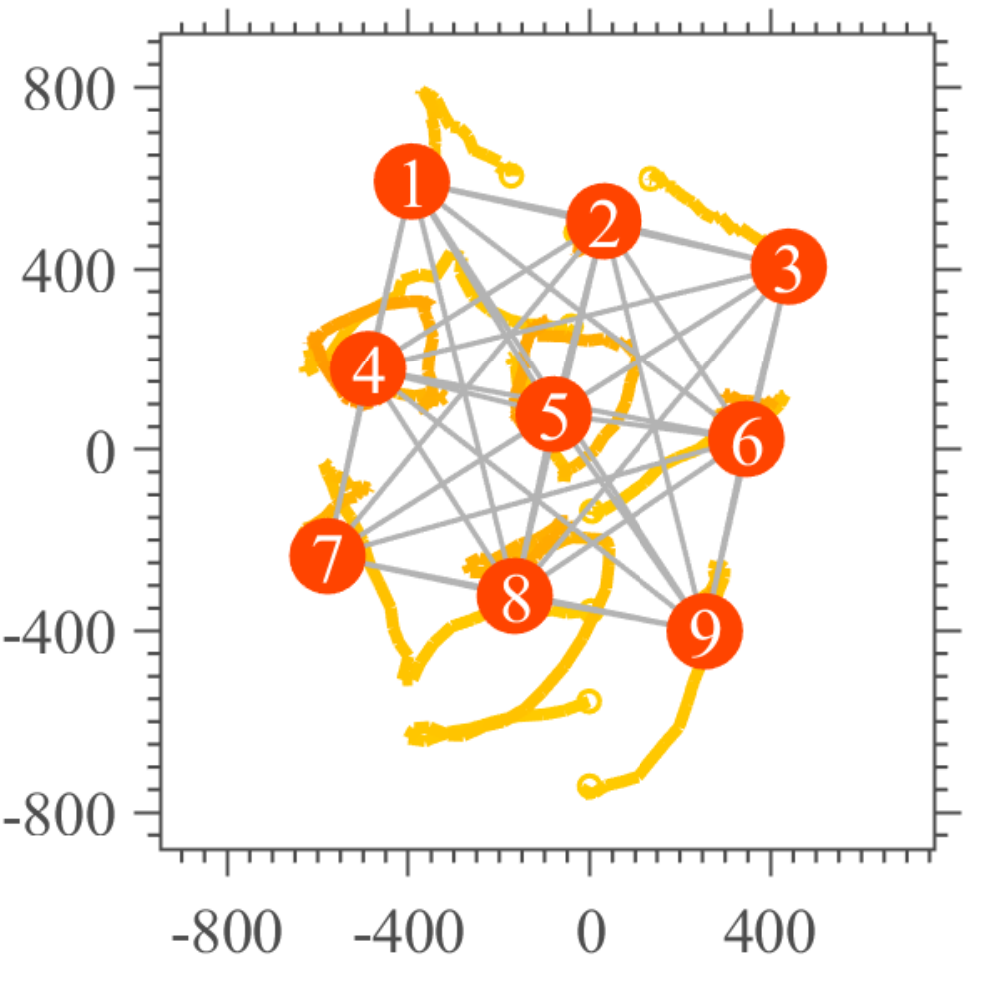}
		\caption{}
	\end{subfigure}%
	\caption{Trajectory of robots estimated from webcam images (a) under single-integrator control, (b) under unicycle control. Units are in millimeter.}
	\label{fig:Grid_Reconst}	
\end{figure}

\begin{figure*}[]
	\centering
	\begin{subfigure}[b]{1.0\textwidth}
		\includegraphics[trim = 25mm 96mm 23mm 93mm, clip, width=1.0\textwidth] {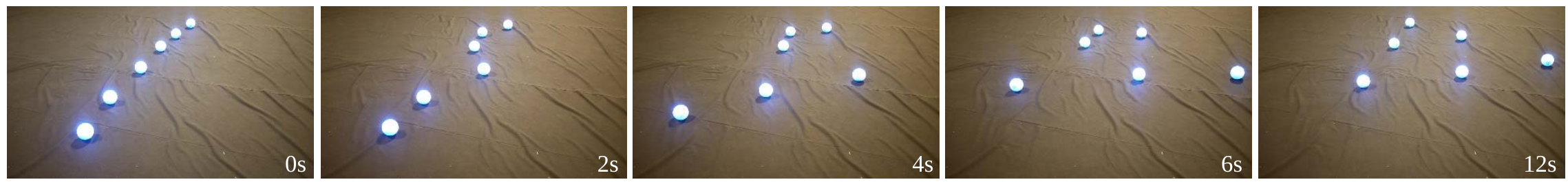}
	\end{subfigure}%
	\caption{Snapshots of experiment for triangle formation at different instances of time.}
	\label{fig:Triangle_Pics}	
\end{figure*}

Our first set of experiments correspond to an equilateral triangle formation with 6 robots.  For this experiment only two computers are used, where the first computer controls robots numbered 1 to 3, and the remaining robots are controlled by the second computer. 
The sensing  topology among the robots is illustrated by gray lines in Fig.~\ref{fig:Triangle_Reconst} and is fixed throughout the experiment. 
Nonzero eigenvalues of the computed gain matrix $A \in  \br^{12 \times 12}$ range from $-1.52$ to $-0.22$. 
Collision avoidance strategy in Algorithm \ref{alg:ColAvoidance}  is used with the activation threshold $d_c = 400$ mm and $r = 100$ mm. The speed of each robot is bounded to $1/3$ of its upper limit, which gives the maximum speed of around $200$ mm/s.

The trajectory of robots under the single-integrator control strategy \eqref{eq:HolonomCtrl} is shown in Fig.~\ref{fig:Triangle_Reconst}(a). These  trajectories are reconstructed from the images provided by the first webcam. The sensing topology among robots is illustrated by gray lines in the figure. 
At their initial position, the robots roughly form a line. Starting from this initial position, they achieve the desired formation as can be further seen from the snapshots of experiment video at different instances of time in Fig.~\ref{fig:Triangle_Pics}.

In a similar experiment with robots starting roughly from the same initial positions, the unicycle control strategy \eqref{eq:kinemCtrl} is used to achieve the desired formation. Estimated trajectory of robots under this control strategy is shown in Fig.~\ref{fig:Triangle_Reconst}(b). 
As the robots get closer to the desired formation, the magnitude of their control vectors become smaller. Once the desired speed is small enough, the floor friction prevents the robots from moving further. This can cause an small steady state error, which can be observed in  Fig.~\ref{fig:Triangle_Reconst}.  
Note that no collisions occur as the robots converge to the desired formation.

\subsection{Hexagon Formation}

\begin{figure*}[]
	\centering
	\begin{subfigure}[b]{1.0\textwidth}
		\includegraphics[trim = 25mm 96mm 23mm 93mm, clip, width=1.0\textwidth] {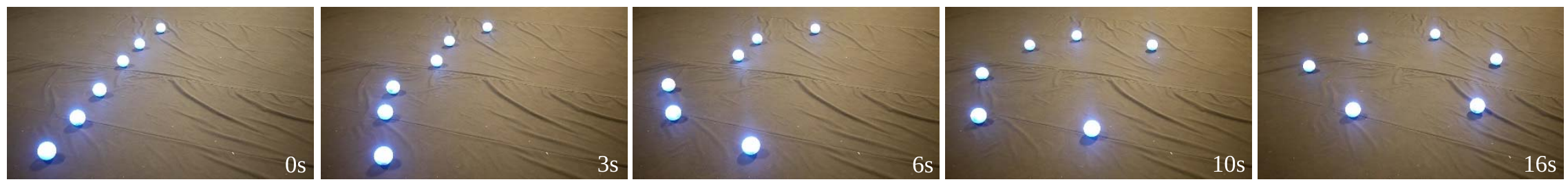}
	\end{subfigure}%
	\caption{Snapshots of experiment for hexagon formation at different instances of time.}
	\label{fig:Hexagon_Pics}	
\end{figure*}

Using the same experimental setup for the triangle formation, we repeat a new set of experiments with the desired formation defined as a hexagon. 
The inter-agent sensing topology is chosen as a cyclic graph, as illustrated by gray lines in Fig.~\ref{fig:Hexagon_Reconst}, and is fixed throughout the experiment. 
The nonzero eigenvalues of matrix $A$ for this desired formation range from $-2$ to $-1$. 
The reconstructed trajectories from webcam images are shown in Fig.~\ref{fig:Hexagon_Reconst}(a) for the single-integrator control, and in Fig.~\ref{fig:Hexagon_Reconst}(b) for the unicycle control. 
Snapshots of the experiment video corresponding to the single-integrator controller are shown in Fig.~\ref{fig:Hexagon_Pics}. As can be seen from the figures, starting from the initial positions, agents converge to the desired formation.

\subsection{Square-Grid Formation}

\begin{figure*}[]
	\centering
	\begin{subfigure}[b]{1.0\textwidth}
		\includegraphics[trim = 25mm 96mm 23mm 93mm, clip, width=1.0\textwidth] {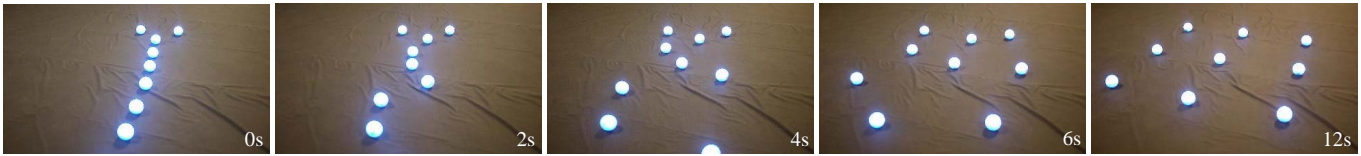}
	\end{subfigure}%
	\caption{Snapshots of experiment for square-grid formation at different instances of time.}
	\label{fig:Grid_Pics}	
\end{figure*}

In our last set of experiments we consider a square-grid desired formation of 9 robots with the sensing topology chosen as a complete graph and fixed throughout the experiment. Here, 3 computers are used to control the robots, where the first computer controls robots numbered 1 to 3, the second computer controls robots 4 to 6, and the third computer controls the remaining robots. 
The parameters used for collision avoidance strategy and the maximum allowed speed of the robots remain the same as previous experiments.

The trajectory of robots reconstructed from the images of the first webcam is shown in Fig.~\ref{fig:Grid_Reconst}(a) for the single-integrator control strategy, and snapshots of experiment video is shown in Fig.~\ref{fig:Grid_Pics}. 
If the distance between two robots is less than the collision avoidance threshold $d_c$, the collision avoidance strategy rotates the control direction outside of the collision cones. However, if the required rotation is more than $\pm 90^\circ$ of the original control direction, the control is set to zero and the robot stops until a feasible direction becomes available. 
The effect of collision avoidance strategy is most notable for robot 2, which is initially surrounded by robots 1, 3, and 4. Consequently, due to the lack of a feasible direction robot 2 remains stationary initially until the surrounding robots move further and a feasible direction becomes available.
Similar experiments are performed by using the unicycle control strategy, where the reconstructed trajectory of robots are shown in Fig.~\ref{fig:Grid_Reconst}(b). 
Due to using different PID gains for the low-level controllers implemented on robots 4 to 9, their trajectories are more distinguished than their corresponding trajectories in Fig.~\ref{fig:Grid_Reconst}(a).

\section{Concluding Remarks and Future Work}
\label{sec:Conclusion}

We presented a distributed formation control strategy for a team of agents with a variety of dynamics to autonomously achieve a desired planar formation. Under the assumption that the sensing graph is undirected and universally rigid, we showed that formation control gains can be designed by solving a SDP problem. This design enjoys several robustness properties, such as robustness to positive scaling and rotation (up to $\pm 90^\circ$) of the control vector, saturations in the input, and switches in the sensing topology.  
The control was extended to agents with higher-order linear (or linearizable) holonomic dynamics, such as quadrotors, followed by further extension to agents with nonholonomic unicycle and car dynamics.
An important outcome of this work was to show that under the proposed control the convergence and robustness guarantees hold for agents with more complex dynamics. Further, a fully distributed collision avoidance algorithm emerged naturally from the robustness properties.
To typify the control, simulations for vehicles with different dynamics were presented, and experiments on a distributed robotic platform where performed.

Future work includes investigating additional requirements, such as inter-agent communication, to guarantee that the collision avoidance algorithm can overcome gridlock scenarios. 
Moreover, inter-agent communication can be exploited in a distributed optimization scheme to solve the SDP problem in a decentralized way. 
Other possible research avenues include formation control of heterogeneous vehicles and time-varying formations.

\section{Acknowledgments}

We would like to thank Giampiero Campa and Danvir Sethi at Mathworks for providing the Matlab's Sphero connectivity package \cite{Sphero} that is used in the proposed robotic platform.


\ifCLASSOPTIONcaptionsoff
  \newpage
\fi

\balance 

\bibliographystyle{IEEEtran}
\bibliography{FromCtrlSphero_Bibs}

\end{document}